\documentclass{article}
\usepackage{one_inc}
\title{
Optimal PAC Bounds Without Uniform Convergence
}
\author{Ishaq Aden-Ali\thanks{Department of Electrical Engineering and Computer Science, UC Berkeley. Email: {adenali@berkeley.edu, yeshwanth@berkeley.edu, shetty@berkeley.edu}} \and Yeshwanth Cherapanamjeri\footnotemark[1] \and Abhishek Shetty\footnotemark[1] \and Nikita Zhivotovskiy\thanks{Department of Statistics, UC Berkeley. Email: zhivotovskiy@berkeley.edu}}

\date{\today}

\begin{document}
\maketitle
\begin{abstract}
In statistical learning theory, determining the sample complexity of realizable binary classification for VC classes was a long-standing open problem.
The results of Simon \cite{simon2015almost} and Hanneke  \cite{hanneke2016optimal} established sharp upper bounds in this setting. 
However, the reliance of their argument on the uniform convergence principle limits its applicability to more general learning settings such as multiclass classification.
In this paper, we address this issue by providing optimal high probability risk bounds through a framework that surpasses the limitations of uniform convergence arguments.

Our framework converts the leave-one-out error of permutation invariant predictors into high probability risk bounds.
As an application, by adapting the one-inclusion graph algorithm of Haussler, Littlestone, and Warmuth \cite{haussler1994predicting}, we propose an algorithm that achieves an optimal PAC bound for binary classification.
Specifically, our result shows that certain aggregations of one-inclusion graph algorithms are optimal, addressing a variant of a classic question posed by Warmuth \cite{warmuth2004optimal}.

We further instantiate our framework in three settings where uniform convergence is provably suboptimal.
For multiclass classification, we prove an optimal risk bound that scales with the one-inclusion hypergraph density of the class, addressing the suboptimality of the analysis of Daniely and Shalev-Shwartz \cite{daniely2014optimal}.
For partial hypothesis classification, we determine the optimal sample complexity bound, resolving a question posed by Alon, Hanneke, Holzman, and Moran \cite{alon2022theory}.
For realizable bounded regression with absolute loss, we derive an optimal risk bound that relies on a modified version of the scale-sensitive dimension, refining the results of Bartlett and Long \cite{bartlett1998prediction}.
Our rates surpass standard uniform convergence-based results due to the smaller complexity measure in our risk bound.
\end{abstract}

\section{Introduction}\label{sec:intro}

The study of the statistical complexity of prediction is a central question in statistical learning theory.
In the simplest setting of realizable prediction, the predication task is as follows: for an instance space $\mc{X}$, label space $\mc{Y}$ and hypothesis class $\mc{H} \subseteq \mc{Y}^{\mc{X}}$, one is given access to $n$ i.i.d.\ labelled training points $S = ((X_1,Y_1), \dots, (X_n,Y_n)) \in (\mc{X} \times \mc{Y})^n$ drawn from a distribution $P$ with the promise that there exists a target hypothesis $f^* \in \mc{H}$ that perfectly labels the data. We refer to $S$ as the \emph{training sample}.
The goal is to use the training sample $S$ to design a predictor, denoted by $\wh{f} (\cdot; S)$, with small prediction error on a held-out sample drawn from the same distribution (which we refer to as the \emph{risk}) with high probability over the samples. Concretely, we would like to minimize
\begin{equation*}
    \E_{(X, Y) \thicksim P} \lsrs{\ell (\wh{f} (X; S), Y)}
\end{equation*}
for a suitable loss function $\ell: \mc{Y} \times \mc{Y} \to [0, 1]$ with probability at least $1 - \delta$ over $S$.

A vast body of literature is dedicated to characterizing the optimal achievable risk as a function of the number data points $n$, the failure probability $\delta$, and the complexity of hypothesis class $\mc{H}$. Even for the simplest setting of binary classification, where $\mc{Y} = \{0, 1\}$, $\ell(\wh{y}, y) = \bm{1} \lbrb{\wh{y} \neq y}$ and the complexity of the function class is captured by its VC-dimension $d$, this remained a challenging open problem. In this specific setting, seminal early works \cite{vapnik1968algorithms, vapnik74theory, blumer1989learnability} established that an Empirical Risk Minimizer (ERM), i.e.,\ any predictor that minimizes the loss on the training samples, achieves the risk bound\footnote{Throughout the paper we use $\log$ to denote the natural logarithm.} 
\begin{equation*}
    \E_{(X, Y) \thicksim P} \lsrs{\ell (\wh{f}_{\mrm{ERM}} (X; S), Y)} \leq C \lprp{\frac{d}{n} \log \lprp{\frac{n}{d}} + \frac{1}{n}\log \lprp{\frac{1}{\delta}}}.
\end{equation*}
High probability risk bounds of this form are usually referred to as Probably Approximately Correct (PAC) bounds, a notion introduced in the classical paper of Valiant \cite{valiant1984theory}.

Meanwhile, a complementary line of work in the \emph{transductive} setting and the related \emph{prediction model of learning} \cite{haussler1994predicting}, led to the development of predictors with strong leave-one-out (LOO) performance with direct implications for the distributional setting. 
In a prototypical example of these settings, there is a fixed realizable sample $S = \lprp{(x_i, y_i)}_{i = 1}^n$ and the learner is given access to the labels $y_i$ of a randomly chosen $n - 1$ of them with the task of predicting accurately on the remaining point. Denoting by $\Si{} = \lprp{(x_j, y_j)}_{j \neq i}$, the goal is to design a predictor $\wh{f}$ which minimizes the LOO error
\begin{equation*}
    \sum_{i = 1}^n \ell \lprp{\wh{f}\lprp{x_i; \Si{}}, y_i}.
\end{equation*}
For binary classification, the famous one-inclusion graph strategy of Haussler, Littlestone, and Warmuth \cite{Haussler1988, haussler1994predicting} achieves LOO error at most $d$.\footnote{Their bound is optimal even up to the leading constant \cite{li2001one}.}
An exchangeability argument yields a predictor $\wh{f}_{\mrm{OIG}}$ with the bound
\begin{equation*}
    \E_{\substack{ (X, Y) \thicksim P \\ S \sim P^n }} \lsrs{\ell \lprp{\wh{f}_{\mrm{OIG}} (X; S), Y}} \leq \frac{d}{n}
\end{equation*}
for the distributional setting \emph{in-expectation} over the training samples. 
Interestingly, this approach eliminates the extra logarithmic factor typically present in bounds for empirical risk minimization, leading to the conjecture that the logarithmic factor might not be necessary, even in the PAC setting.
Further, it suggests that the one-inclusion graph predictor could be optimal \cite{warmuth2004optimal}. 
Since then, these predictors have been adapted for a variety of contexts where uniform convergence, a crucial aspect needed for analyzing ERM predictors, is suboptimal or fails to hold.
Examples include multiclass learning \cite{rubinstein2009shifting,simon2010one,daniely2014optimal}, partial hypothesis learning \cite{alon2022theory}, and bounded regression \cite{bartlett1998prediction}. 

In a breakthrough result, Hanneke \cite{hanneke2016optimal}, building upon the work of Simon \cite{simon2015almost}, devised an optimal PAC predictor $\wh{f}_{\mrm{OPT}}$ satisfying
\begin{equation*}
      \E_{(X, Y) \thicksim P} \lsrs{\ell (\wh{f}_{\mrm{OPT}} (X; S), Y)} \leq C \lprp{\frac{d}{n} + \frac{1}{n} \log \lprp{\frac{1}{\delta}}},
\end{equation*}
affirming the conjectured optimal rate.
Remarkably, the analysis does not leverage the optimal in-expectation performance of the leave-one-out predictors but instead considers a careful aggregation of suboptimal ERM predictors.
Unfortunately, the use of ERM precludes the application of this approach in settings where uniform convergence does not hold, but accurate prediction is nevertheless possible as evidenced by the success of leave-one-out predictors.

In this context, our main contribution (\cref{thm:main}) is a general technique that transforms a broad class of predictors with optimal leave-one-out performance, which satisfy strong \emph{in-expectation} guarantees, into ones with optimal PAC bounds. 
On a conceptual level, our work addresses the classical conjecture of Warmuth \cite{warmuth2004optimal} which hypothesized the optimal rate was achievable via the one-inclusion-graph algorithm. 
While the conjecture in the strict sense is false as there exist one-inclusion strategies which are optimal in-expectation but perform poorly in the PAC model \cite{adenali2022oneinclusion}, we show that a simple aggregation over one-inclusion strategies constructed on prefixes of the training data suffices to restore optimal performance in the PAC setting. 
This is formalized in our main result presented as \cref{thm:main}.

More concretely, our reliance on the existence of optimal leave-one-out predictors allows us to completely bypass uniform convergence, enabling applications to a range of settings where uniform convergence may not hold.
Consider the general realizable prediction setting with a predictor, $\wh{f}_{\mrm{LOO}}$, satisfying 
\begin{equation*}
    \sum_{i = 1}^n \ell \lprp{\wh{f}_{\mrm{LOO}} (x_i; \Si), y_i} \leq M_n
\end{equation*}
for any realizable sample $S$.
Our framework now yields a predictor, $\wh{f}_{\mrm{PAC}}$ satisfying
\begin{equation*}
    \E_{(X, Y) \thicksim P} \lsrs{\ell \lprp{\wh{f}_{\mrm{PAC}} \lprp{X; S}, Y}} \leq C \lprp{\frac{M_n}{n} + \frac{1}{n} \log \lprp{\frac{1}{\delta}}}.
\end{equation*}
for the distributional setting.

In several cases of interest, there exist leave-one-out predictors which achieve optimal in-expectation performance for the first term.
Our results show that these may be extended to the PAC setup paying only \emph{additively} in $\log (1 / \delta)$ which is also known to be optimal.
We now illustrate some applications of our result.
\begin{enumerate}
    \item \textbf{Multiclass Classification:} There exist leave-one-out predictors with error scaling with the graph density, $\dens{}$, of the associated one-inclusion hypergraph leading to optimal in-expectation performance. Our results now immediately yield an optimal PAC bound in this setting (\cref{thm:risk_multiclass}). Moreover, in the context of binary classification, we recover the optimal PAC bound of Hanneke \cite{hanneke2016optimal} with a simplified analysis, improving the constant factors by several orders of magnitude.
    \item \textbf{Learning Partial Hypotheses:} Here, the optimal bound on $M_n$ scales with the VC-dimension of the \emph{partial} hypothesis class, a generalization of the VC-dimension in the binary setting. Our result yields an optimal PAC bound answering an open problem in the literature \cite{alon2022theory}. See \cref{cor:risk_partial}.
    \item \textbf{Bounded Regression:} In the setting of bounded $\ell_1$ regression, we obtain the first predictor in the PAC setting whose error depends on the scale-sensitive complexity measure $\fatvf{\mc{H}} (\gamma)$ which in some cases represents a significant improvement over those obtained by the classic uniform convergence-based arguments which depend instead on the \emph{fat-shattering} dimension $ \fat{\mc{H}}(\gamma)$. 
    See \cref{ssec:app_regression,thm:loreg}.
\end{enumerate}

The remainder of the paper is organized as follows. 
The main results of our paper, including a formal description of our framework and its application to several settings of interest are discussed in \cref{sec:results}. We present related work in \cref{sec:related_work}. Subsequently, we prove our main result in \cref{sec:proof_main_result}.

\section{Main results}
\label{sec:results}

\subsection{Notation}
\label{ssec:notation}

We use $\mc{X}$ to denote our instance space, $\mc{Y}$ to denote the label space, and sometimes use $\mc{Z} = \mc{X} \times \mc{Y}$.
Let $\mc{U} = \bigcup_{n=1}^{\infty} \mc{Z}^n$ denote the set of possible observable training samples. A predictor $\widehat{f}: \mc{X} \times \mc{U} \to \mc{Y}$ takes as input a test point and training sample and outputs a prediction. For a training sample $S = ((X_1, Y_1), \dots, (X_n,Y_n))$, we will use $\Si$ to denote the training sample with the $i^{th}$ data point removed, i.e., $\Si = ((X_1, Y_1), \dots, (X_{i-1},Y_{i-1}), (X_{i+1},Y_{i+1}), \dots, (X_n,Y_n))$ for all $i \in [n]$. 
For $j\in \left[ n  \right]$, we will use $S_{\leq j}$ to denote the training sample 
$((X_1, Y_1), \dots, (X_j,Y_j))$.
We will also find it convenient at times to write $Z_i = (X_i, Y_i)$ for the $i^{th}$ point in a training sample $S$.
Throughout the paper, we will use the upper case $X$, $Y$ and $Z$ to denote random variables and lower case $x$, $y$ and $z$ to denote the realizations.

In the learning setting, the algorithm is given access to samples from an unknown distribution $P$ over $\mc{Z} = \mc{X} \times \mc{Y}$.
The goal of the learning algorithm is to produce a hypothesis $\widehat{f}$ that has low \emph{risk} under $P$ with respect to a \emph{loss function} $\ell : \mc{Y} \times \mc{Y} \to [0, 1]$, where we define the risk to be $\risk{\widehat{f}}{P} \coloneqq \E_{(X,Y) \sim P}[\ell(\widehat{f}(X), Y)]$.
Throughout this paper we will make the assumption that the unknown distribution $P$ is \emph{realizable} by $\mc{H}$.
This means there is some $f^* \in \mc{H}$ such that for $Y = f^*(X)$ almost surely.
Similarly, we say a training sample $S = ((x_1, y_1), \dots, (x_n,y_n))$ is realizable by $\mc{H}$ (or simply realizable) if there is some $f^* \in \mc{H}$ such that $f^*(x_i) = y_i$ for all $(x_i, y_i) \in S$.
Given a training sample $S = ((x_1,y_1),\dots,(x_n,y_n))$, we will find it convenient to define $\uniS \subseteq \mc{X}$ to be the \emph{set} of \emph{unique} $x_i$ terms in the training sample $S$.

\subsection{Assumptions}
\label{ssec:assumptions}

Our main result is applicable to a broad class of predictors satisfying two natural assumptions. 
The main assumption pertains to the \emph{leave-one-out} error.
As described in \cref{sec:intro}, the predictor is first evaluated on each point in the training sample with the predictor trained on the training sample with that point removed. The sum of these errors over all points in the training sample yields the leave-one-out error. 
We assume that the predictor has bounded leave-one-out error.  

\begin{assumption}
    \label{as:bdd_loo}  The predictor $\widehat{f}$ has bounded leave-one-out error for any realizable sample. Formally, for any realizable $S = \lprp{(x_i, y_i)}_{i = 1}^n$, we have
    \begin{equation*}
    \sum_{i = 1}^{n} \ell (\widehat{f}(x_i; \Si{}),y_i) \leq M_n.
    \end{equation*}
    Further, we will assume that $M_n$ is monotone in the sample size $n$, that is, $M_n \le M_{n+1}$ for all $n \in \mb{N}$.\footnote{This is largely without loss of generality since in most natural settings, we can replace $M_n $ with $\max_{i\leq n} M_i$ to enforce monotonicity.}
\end{assumption}

\begin{remark} \label{rem:loo_to_risk}  
    Note that by averaging over a sample $S$ of size $n$ drawn from $P$, the leave-one-out error from \cref{as:bdd_loo} translates to an expected risk bound of
    \begin{align}
        \E_{S \sim P^n } \risk{\widehat{f}\left(\cdot; S  \right)}{P}  \leq \frac{ M_{n+1} }{n+1}. 
    \end{align}
\end{remark}

The second assumption necessitates that our classifier is symmetric with respect to the training data.
This is intuitive in the batch setting where the training data is drawn i.i.d.\ from a distribution and is satisfied by numerous algorithms in both theoretical and practical contexts. 
We formally state it as the following assumption.

\begin{assumption}
    \label{as:perm_inv} The predictor $\widehat{f}$ is symmetric in the training data: for all $n \in \mb{N}$, for any permutation $\pi$, $S \in \mc{Z}^n, x \in \mc{X}$, we have
    \begin{equation*}
        \widehat{f}(x; S) = \widehat{f}(x; S_\pi),
    \end{equation*}
    where $S_\pi$ is the sample $S$ permuted by $\pi$.
\end{assumption}

The main motivation to consider these assumptions is the \emph{one-inclusion graph algorithm} which is symmetric and, as previously discussed, has optimal leave-one-out error in various settings.
For example, in the case of realizable binary classification with hypothesis classes with finite VC-dimension, this algorithm has $M_n$ that is independent of the sample size. We will discuss our applications subsequently and the one-inclusion graph algorithm in detail in \cref{app:oig}.

\subsection{Main result}
\label{ssec:main_result}
    We present our main result which concerns the performance of a sequence of hypotheses generated by training a predictor, $\wh{f}$, on prefixes of the training data, $(\Sli[t])_{t = 1}^n$. 
    When $\wh{f}$ is symmetric and has bounded leave-one-out error, our result asserts that the \emph{average} risk of the predictors in the \emph{suffix} of the sequence $(\widehat{f}(\cdot ; \Slt ))_{t = 1}^{n}$ is bounded by $M_n$, with high probability. 
    For the sake of simplicity, we assume for the rest of the paper that $n/4$ is an integer.

\begin{restatable}{thm}{main}\label{thm:main}
    Fix a hypothesis class $\mc{H} \subseteq \mc{Y}^{\mc{X}}$, loss function $\ell: \mc{Y} \times \mc{Y} \to [0, 1]$, and a predictor $\widehat{f} : \mc{X} \times \mc{U} \to \mc{Y}$ satisfying \cref{as:perm_inv,as:bdd_loo}. 
    Then, for any realizable distribution $P$ over $\mc{X} \times \mc{Y}$ and confidence parameter $\delta \in (0,1)$,
    given a training sample $S \sim P^n$ we have
    \begin{equation*}
        \frac{\splita }{\splitb n} \cdot \sum_{t = n / \splita}^{n-1} \risk{\widehat{f}(\cdot; \Slt{})}{P} \leq \cishaq \lprp{ \frac{M_n}{n} + \frac{1}{n}\log \left(\frac{2}{\delta}\right)},\label{eq:main_risk_bound}
    \end{equation*}
    with probability  at least $1-\delta$ over the randomness of $S$.
\end{restatable}

Our proof technique is based on \emph{online-to-batch conversion}, which transforms any online learning algorithm --- making $M$ mistakes on a finite sample --- into a batch algorithm utilizing the entire sample in a statistical setting.
A key technical aspect of our work is demonstrating that this technique remains applicable even if $M$ is a random variable dependent on the specific realization of the training sample.

Another technical element of our proof is the use of suffix averaging.
To illustrate this, consider an algorithm that, upon processing an i.i.d. realizable sample of size $t$, makes an error on a newly sampled point with probability at most $d/t$, where $d$ is, for example, the VC dimension. The expected total number of mistakes for a sample of size $n$ is then bounded by $\sum\nolimits_{t = 1}^n d/t = O(d\log n)$, which cannot lead to the desired accuracy. 
Interestingly, when utilizing only the suffix of the sample, we have $\sum\nolimits_{t = n/2}^n d/t = O(d)$, effectively eliminating the logarithmic factor. Using martingale arguments, we essentially convert the last computation into a high probability risk bound that scales as
\[
O\left(\frac{d}{n} + \frac{1}{n}\log\left(\frac{1}{\delta}\right)\right).
\] 
Remarkably, the two observations discussed above emerged almost simultaneously in the late 1980s. Haussler, Littlestone, and Warmuth \cite{Haussler1988} discovered that suffix averaging does not introduce logarithmic factors in sequential errors, while Littlestone established optimal online-to-batch conversions for algorithms with a deterministic number of errors \cite{Littlestone89}. \cref{thm:main} connects these two insights.

\subsection{Classification}
\label{ssec:app_classification}
The first major application of \cref{thm:main} is to the setting of multiclass classification. 
In this setting, we assume the label space $\mc{Y}$ is finite and we take the loss function to be the $0/1$ loss. In this section, we use $\err{f}{P}$ to denote the prediction risk of $f$.
Before we state the main result, we introduce \emph{the one-inclusion hypergraph density} which is the central complexity measure we use.
To do so, we introduce some basic definitions and notation.

Let $G = (V,E)$ be a hypergraph. 
Given a subset of vertices $U \subseteq V$, we define the induced hypergraph $G[U]$ to be the hypergraph with vertex set $U$ and edge set $E' = \{e\cap U: e \in E,\ |e \cap U| \ge 1\}$.
For a hypergraph $G = (V,E)$, we define the density of $G$ to be 
\[
\operatorname{Dens}(G) = \frac{1}{|V|}\sum_{e \in E}\left(|e|-1\right).
\]
We define the maximum density of $G$ to be
\[
\mu(G) = \max_{U \subseteq V} \  \operatorname{Dens}(G[U]).
\]
In words, $\mu(G)$ is the largest density of an induced hypergraph of $G$.
We will also find it useful to define the average degree of a hypergraph $G = (V,E)$ to be
\[
\avgd(G) = \frac{ 1 }{|V|}\sum_{e \in E : |e| > 1}|e|.
\]

Given a hypothesis class $\mc{H} \subseteq \mc{Y}^{\mc{X}}$ and subset $\sub = \{x_1, \dots, x_n\} $ of the instance space $\mc{X}$, we define the projection of $\mc{H}$ onto $\sub$ to be $\mc{H}|_{\sub} = \{(f(x_1) , \dots, f(x_n)) : f \in \mc{H} \}$.
For a projection $\mc{H}|_{\sub}$ onto an $n$-element set $U$, we will often write $f(i) = f(x_i)$ where $i \in [n]$.

We can now define one-inclusion hypergraphs which were first studied in the setting of binary classification by Haussler, Littlestone, and Warmuth~\cite{haussler1994predicting},\footnote{In the binary setting, the one-inclusion hypergraph is simply a graph.} and generalized to the multiclass setting by Rubinstein, Bartlett, and Rubinstein~\cite{rubinstein2009shifting}.
Daniely and Shalev-Shwartz~\cite{daniely2014optimal} initiated the study of these objects for infinite $\mc{Y}$.
\begin{definition}[One-inclusion hypergraph]
Fix a hypothesis class $\mc{H} \subseteq \mc{Y}^{\mc{X}}$ and an $n$-element subset of the domain $\sub \subseteq \mc{X}$.
The one inclusion hypergraph $G(\mc{H}|_{\sub}) = (V,E)$ has its vertex set as $V \coloneqq \mc{H}|_{\sub}$.
For every $i \in [n]$ and $f : [n] \setminus \{i\} \to \mc{Y}$, let $e_{i,f}$ be the set of all $f' \in \mc{H}|_{\sub}$ such that $f'(j) = f(j)$ for all $j \in [n] \setminus \{i\}$.
The edge set of the one-inclusion graph $G(\mc{H}|_{\sub})$ is given by
\[
E = \{ e_{i,f} : i \in [n],\ f : [n] \setminus \{i\} \to \mc{Y}, \  |e_{i,f}| > 0\}.
\]
In words, we create a hyperedge for all vertices (projected hypotheses) that are the same in all but one coordinate $i \in [n]$, and we do this for every coordinate.
\end{definition}
The one-inclusion hypergraph density of a hypothesis class $\mc{H} \subseteq \mc{Y}^{\mc{X}}$ is defined as follows.\footnote{Our definition of this complexity measure is slightly different than the original definition introduced by Daniely and Shalev-Shwartz \cite{daniely2014optimal}. In their paper, they defined this complexity measure using the \emph{average degree} instead of the \emph{edge density}, which differ by a factor of at most $2$. Their definition is also more general and handles infinite $\mc{Y}$.}
\begin{definition}[One-inclusion hypergraph density of $\mc{H}$~\cite{daniely2014optimal}]\label{def:graph_dens}
Fix a hypothesis class $\mc{H} \subseteq \mc{Y}^{\mc{X}}$.
We define the \emph{one-inclusion hypergraph density} of the class $\mc{H}$ for the sample size $n$ to be
\[
\dens = \max_{S \in \mc{Z}^n} \ \mu(G(\mc{H}|_{\uniS})).
\]
In words, $\dens$ is the largest maximum density of any one-inclusion graph formed from $\mc{H}$ and a training sample $S$ of size $n$ (possibly containing repetitions).
\end{definition}

Haussler, Littlestone, and Warmuth~\cite{haussler1994predicting} defined the one-inclusion \emph{graph} algorithm for the setting of binary classification, and Rubinstein, Bartlett, and Rubinstein~\cite{rubinstein2009shifting} generalized this algorithm to the multiclass setting and dubbed it the one-inclusion \emph{hypergraph} algorithm. 
The one-inclusion hypergraph algorithm satisfies \cref{as:perm_inv,as:bdd_loo}, and it is the predictor that we use in our results.
The details of this algorithm are not essential to prove the new results in this paper, however we include a description in \cref{app:oig} for the interested reader.
Finally, we note that the one-inclusion hypergraph algorithm collapses to the original one-inclusion graph algorithm when $\mc{Y} = \{0,1\}$.

Our main PAC risk bound for multiclass classification scales with $\dens$.
\begin{restatable}{thm}{riskmulticlass}\label{thm:risk_multiclass}
Fix a hypothesis class $\mc{H} \subseteq \mc{Y}^{\mc{X}}$. 
There is a predictor $\wh{f} : \mc{X} \times \mc{U} \to \mc{Y}$ which, for any $\delta \in (0, 1)$, and $S \ts P^n$ sampled from any realizable distribution $P$, satisfies
\[
\err{\widehat{f}(\cdot ; S)}{P} \le \twocishaq\left(\frac{\lceil\dens\rceil}{n}+\frac{1}{n}\log\left(\frac{2}{\delta}\right)\right),
\]
with probability at least $1-\delta$ over the randomness of $S$.
\end{restatable}
We include the proof of this result in \cref{app:oig}. 
This bound provides a sharp characterization of the prediction risk in terms of $\dens$, boosting the \emph{optimal in-expectation} risk bound of Daniely and Shalev-Shwartz~\cite{daniely2014optimal} to the canonical PAC setup. In particular, our optimal PAC bound does not exhibit an \emph{explicit} dependence on the number of classes which is \emph{not} the case for the bounds achieved by ERM classifiers based on uniform convergence. We refer to Section \ref{sec:related_work} for additional details.

A characterization of learnability with respect to the sequence $(\dens)_{n \in \mathbb{N}}$ is a bit unsatisfying due to its implicit dependence on $n$.
This prompted Daniely and Shalev-Shwartz to investigate whether there is a \emph{single number} that can replace this sequence of complexity measures in these risk bound.
They proposed a notion of a dimension which was later named the Daniely Shalev-Shwartz (DS) dimension \cite{brukhim2022characterization}.\footnote{This dimension was originally defined to allow for infinite $\mc{Y}$.}
\begin{definition}[DS dimension \cite{daniely2014optimal}]
Fix a hypothesis class $\mc{H} \subseteq \mc{Y}^{\mc{X}}$. The \emph{Daniely Shalev-Shwartz dimension} of $\mc{H}$, denoted $\DS$, is the largest integer $d$ such that there is a subset $\mc{H}' \subseteq \mc{H}$ satisfying
\[
\max_{S \in \mc{Z}^d}\avgd(G(\mc{H}'|_{\uniS})) = d.
\]
\end{definition}
To better appreciate this definition, recall the definition of the VC dimension in the binary setting.
Fix a binary hypothesis class $\mc{H} \subseteq \{0,1\}^{\mc{X}}$.
We say an $n$-element set $\sub \subseteq \mc{X}$ is \emph{shattered} by $\mc{H}$ if $ \mc{H}|_{\sub} = \{0,1\}^n$.
The VC dimension of $\mc{H}$ is the largest integer $d$ such that there exists a $d$-element subset of $\mc{X}$ that is shattered by $\mc{H}$.
When $\mc{Y} = \{0,1\}$, the DS dimension is equivalent to the VC dimension.
Daniely and Shalev-Shwartz made the following conjecture that relates the DS dimension to the sequence $(\dens)_{n \in \mathbb{N}}$.
\begin{conjecture}[\cite{daniely2014optimal}]\label{eq:dsconjecture}
Fix a hypothesis class $\mc{H} \subseteq \mc{Y}^{\mc{X}}$ with DS dimension $\DS$.
There is an absolute constant $c > 0$ such that for any $n \ge \DS$ we have
\[
    \dens \le c\cdot d_{\operatorname{DS}}.
\]
\end{conjecture}
A positive resolution of this combinatorial conjecture would imply, via \cref{thm:risk_multiclass}, the PAC bound 
\[
  \err{\widehat{f}}{P} = O\left(\frac{d_{\operatorname{DS}}}{n}+\frac{1}{n}\log\left(\frac{1}{\delta}\right)\right),
\]
which would significantly improve upon the recent work of Brukhim, Carmon, Dinur, Moran, and Yehudayoff~\cite{brukhim2022characterization} who proved the bound\footnote{Here, the tilde notation hides poly-logarithmic factors in the parameters $n$ and $\DS$.}
\[
\widetilde{O}\left(\frac{d_{\operatorname{DS}}^{3/2}}{n} + \frac{1}{n}\log\left(\frac{1}{\delta}\right)\right).
\]

The main motivation for \cref{eq:dsconjecture} 
is an elegant result of Haussler, Littlestone, and Warmuth \cite{haussler1994predicting} (see also \cite{haussler1995sphere}) which shows that for every \emph{binary} hypothesis class $\mc{H}$ with VC dimension $d$, the one-inclusion hypergraph density is at most $d$.
\begin{theorem}[\cite{haussler1994predicting, haussler1995sphere}]
    \label{thm:hlw_density}
Fix a hypothesis class $\mc{H}\subseteq \{0,1\}^\mc{X}$ with VC dimension $d$. For every $n \ge 1$ we have
    \[
        \dens \leq d.
    \]
\end{theorem}

In fact, using \cref{thm:risk_multiclass,thm:hlw_density} in the binary setting, we can get a sharp PAC bound that scales with a single parameter, the VC dimension.  
\begin{corollary}\label{cor:risk_binary}
Fix a hypothesis class $\mc{H} \subseteq \{0,1\}^{\mc{X}}$ with VC dimension $d$. 
There is a predictor $\wh{f} : \mc{X} \times \mc{U} \to \{0,1\}$ which, for any $\delta \in (0, 1)$, and $S \ts P^n$ sampled from any realizable distribution $P$, satisfies
\[
\err{\widehat{f}(\cdot ; S)}{P}  \le \twocishaq\left(\frac{d}{n}+\frac{1}{n}\log\left(\frac{2}{\delta}\right)\right),
\]
with probability at least $1-\delta$ over the randomness of $S$.
\end{corollary}
\cref{cor:risk_binary} recovers the known optimal risk upper bound for binary classification first proven by Hanneke~\cite{hanneke2016optimal}.
Recently, Larsen~\cite{larsen2022bagging} showed that an implementation of the natural bagging heuristic also achieves an optimal risk bound.
Our proof of the optimal bound is remarkably simpler than both of their proofs.
Furthermore, both Hanneke and Larsen state that the constant factors in their upper bounds are very large and explicitly ask whether these constants can be reduced.
Our new analysis reduces the constant factors by a few orders of magnitude.

In addition, our result can be seen as a partial answer to a question of Warmuth \citep{warmuth2004optimal} who asked whether the one-inclusion graph algorithm can achieve an optimal PAC risk bound. 
Recent work \cite{adenali2022oneinclusion} shows that the conjecture is not true in the strict sense.
That is, the one-inclusion graph algorithm can have a large risk bound in the PAC setting.
Our result shows that, though the one inclusion predictor itself is not optimal, a simple aggregation of such predictors can achieve an optimal PAC bound.

The final classification setting we apply our result to is partial hypothesis classification.
In this setting, the label set is $\mc{Y} = \{0, 1, \star\}$ {where $\star$ label is interpreted as an ``I don't know'' response}.
The loss function $\ell$ is again the $0/1$ loss.
We can extend the notion of shattering and VC dimension used for binary hypothesis classes to partial hypothesis classes.
We say an $n$-element set $\sub = \{x_1, \dots, x_n\} \subseteq \mc{X}$ is \emph{shattered} by a partial hypothesis class $\mc{H} \subseteq \{0,1,\star\}^{\mc{X}}$ if $\{0,1\}^n \subseteq \mc{H}|_{\sub}$.
The VC dimension of a partial hypothesis class $\mc{H}$ is now defined as the largest integer $d$ such that there exists a $d$-element subset of $\mc{X}$ that is shattered by $\mc{H}$.\footnote{Note that the complexity of a class does not depend directly on $\star$ labels distinguishing this setting from just being an instance of the multiclass setting with three labels.}
We say that a distribution $P$ over $\mc{X} \times \{0,1\}$ (importantly, not over $\mc{X} \times \{0,1, \star\}$) is realizable by $\mc{H}$ if there is a target hypothesis $f^* \in \mc{H}$ such that $\err{f^*}{P} = 0$. Note that this definition is slightly different from the definition of the realizability we used above.

In this setting, the predictor we use is a simple modification of the one-inclusion hypergraph algorithm to exclude the $\star$ labels.
One can show, as in the binary classification setting (\cref{thm:hlw_density}), that the one-inclusion hypergraph density $\dens$ of $\mc{H}$ is at most $d$, the VC dimension of $\mc{H}$.
This leads to the following risk bound which implies a sample complexity upper bound that matches the lower bound in \cite{alon2022theory} up to universal constants.
This resolves a question posed in \cite{alon2022theory}.\footnote{The definition of realizable learning used in \cite{alon2022theory} is more general than the standard notion that we use here. This difference is immaterial in our analysis and our bounds hold under their more general definition as well.}
We elaborate more on this simple modification of the one-inclusion hypergraph in \cref{app:oig_partial}.
\begin{corollary}\label{cor:risk_partial}
Fix a partial hypothesis class $\mc{H} \subseteq \{0,1,\star\}^{\mc{X}}$ with VC dimension $d$.
There is a predictor $\wh{f} : \mc{X} \times \mc{U} \to \{0,1,\star\}$ which, for any $\delta \in (0, 1)$, and $S \ts P^n$ sampled from a realizable distribution $P$, satisfies
\[
\err{\widehat{f}(\cdot ; S)}{P} \le \twocishaq\left(\frac{d}{n}+\frac{1}{n}\log\left(\frac{2}{\delta}\right)\right) ,
\]
with probability at least $1-\delta$ over the randomness of $S$.
\end{corollary}
In addition to resolving these questions regarding optimal rates, our technique can be seen as a unifying perspective for these notions of classification. This should be contrasted with the fact that the previous algorithms and proof of optimality in the binary case \cite{hanneke2016optimal,larsen2022bagging} were built on empirical risk minimization which can provably fail in the more general settings.
Due to the generality of our main result, the list of potential applications in this section is not exhaustive.
For instance, one could extend \cref{cor:risk_partial} to the multiclass partial hypothesis setting. 
This would improve some recent bounds appearing in \cite{kalavasis2022multiclass}.

\subsection{Bounded regression}
\label{ssec:app_regression}
Moving from a discrete set $\mathcal Y$, we now explore bounded regression with $\mathcal Y = [0, 1]$ and the absolute loss $\ell(\widehat{y}, y) = \abs{\widehat{y} - y}$. The key reference for this section is the work of Bartlett and Long \cite{bartlett1998prediction}, who observed the applicability of the one-inclusion graph algorithm in this setting. 
We begin by revisiting some standard definitions. Let $\mc{X}$ be an instance space with label space $\mc{Y} = [0, 1]$ and $\mc{H} \subseteq \mc{Y}^{\mc{X}}$ denote a class of real-valued functions mapping $\mc{X}$ to $\mc{Y}$. Let $\gamma \ge 0$ be a margin parameter.
Define $\fatvf{\mathcal H}(\gamma)$ to be the largest integer for which there exists $\tau \in [0, 1]$ and a subset $S \subseteq X$ such that $|S| = \fatvf{\mathcal H}(\gamma)$ and, for any $A \subseteq S$, there is an $f_{A} \in \mathcal H$ satisfying $f_A(x) \ge \tau + \gamma$ for all $x \in A$ and $f_A(x) \le \tau - \gamma$ for $x \in S\setminus A$. 
This complexity measure, which is a scale-sensitive version of the original shattering notion of Vapnik and Chervonenkis \cite{vapnik1971theory} for real-valued functions, is explicitly defined in \cite{alon1997scale} with the name $V_{\gamma}$-dimension. 
We refer to \cite{kleer2023primal} where this complexity measure is discussed in detail.

A real-valued predictor using the one-inclusion graph algorithm is constructed in the following lemma.
Unfortunately, the authors were unable to verify the correctness of the proof of Theorem 1 in \cite{bartlett1998prediction} where such a construction was originally proposed.
We instead devise an alternative predictor that achieves an even sharper in-expectation risk bound than the one claimed in \cite{bartlett1998prediction}.
The proof of the lemma may be found in \cref{ssec:proof_oiglo}.
\begin{restatable}{lem}{oiglo}\label{lem:oiglo}
    Fix a hypothesis class $\mc{H} \subseteq [0, 1]^\mc{X}$. There is a predictor $\widehat{f}: \mc{X} \times \mc{U} \to \mc{Y}$ that, for any realizable sample $S = {(x_i, y_i)}_{i = 1}^n$, satisfies \cref{as:perm_inv,as:bdd_loo} with a leave-one-out error bound
    \begin{equation*}
        M_n \leq (n + 1) \gamma + \fatvf{\mc{H}} (\gamma).
    \end{equation*}
\end{restatable}

The main result of this section is the following PAC bound.
\begin{restatable}{thm}{loreg}\label{thm:loreg}
    Fix hypothesis class $\mc{H} \subseteq [0, 1]^\mc{X}$ and margin $\gamma \in (0, 1)$. 
    There is a predictor $\wh{f} : \mc{X} \times \mc{U} \to [0,1]$ which, for any $\delta \in (0, 1)$, and $S \ts P^n$ sampled from any realizable distribution $P$, satisfies
    \begin{equation*}
        \E_{(X,Y) \thicksim P} \lsrs{\abs{\widehat{f}(X ; S) - Y}} \leq \cishaq \lprp{\gamma + \frac{\fatvf{\mc{H}} (\gamma)}{n} + \frac{1}{n}\log\left(\frac{2e}{\delta}\right)},
    \end{equation*}
    with probability at least $1-\delta$ over the randomness of $S$.
\end{restatable}
\cref{thm:loreg} presents a bound that is optimal in the following sense: Bartlett and Long \cite{bartlett1998prediction} demonstrated a certain hypothesis class $\mathcal{H}$ for which the in-expectation version of \cref{thm:loreg} is optimal up to multiplicative constant factors.

\subsubsection{Relations to uniform convergence results}
\label{sec:relationstouniformconvergence}
Characterization of uniform convergence for real valued sets of functions has been a central question in empirical processes theory motivated by problems in statistical learning theory and convex geometry.
The work of Vapnik and Chervonenkis \cite{vapnik1971theory} initiated the study of uniform convergence for real valued functions in terms of a scale-insensitive version of the complexity measure $\fatvf{\mathcal H}(\gamma)$.
The work of Alon et al. \cite{alon1997scale} studied the complexity measure $\fatvf{\mathcal H}(\gamma)$ and showed that its finiteness for all $\gamma > 0$ is a necessary and sufficient condition for a class of functions to be a uniform Glivenko-Cantelli (GC) class. 
Similar conclusions appeared in  \cite{talagrand1996glivenko, talagrand2003vapnik}. 
See also the related results in \cite{vapnik1982necessary}.

The main problem with the complexity measure $\fatvf{\mathcal H}(\gamma)$ in the context of the uniform GC property is that it poorly characterizes the covering numbers of the corresponding class. According to  Talagrand \cite{talagrand2003vapnik}, using this complexity measure in quantitive bounds leads to a \say{loss of accuracy that is devastating when dealing
with precise rates}. In fact, most of literature works with a different complexity measure that we call the $P_{\gamma}$-dimension. Let $\mathcal{H}$ be a class of $[0, 1]$-valued functions on some domain set $\mathcal X$ and let $\gamma$ be a margin. We say that $\mathcal{H}$ $P_\gamma$-shatters a set $S \subseteq \mathcal X
$ if there exists a function $s: S \to [0, 1]$ such that, for every $A \subseteq S$, there exists some $f_A \in \mathcal{H}$ satisfying: For every $x \in S \setminus A$, $f_A(x) \leq s(x) - \gamma$ and, for every $x \in A$, $f_A(x) \geq s(x) + \gamma$. Let the $P_\gamma$-dimension of $\mathcal{H}$, denoted by $\fat{\mathcal H}(\gamma)$, be the maximal integer such that the $S \subseteq X$ of size $\fat{\mathcal H}(\gamma)$ is $P_{\gamma}$-shattered by $\mathcal{H}$. 
This complexity measure is typically referred to as the \emph{fat-shattering} dimension \cite{kearns1994efficient, bartlett1998prediction, anthony1999neural}. The main distinction between complexity measures $\fatvf{\mathcal H}(\gamma)$ and $\fat{\mathcal H}(\gamma)$ lies in the level at which shattering occurs. In the case of $\fatvf{\mathcal H}(\gamma)$, shattering takes place at a single level, denoted by $\tau$. On the other hand, the introduction of the function $s$ in the second definition allows for shattering at various scales, represented by $\tau_i$, for distinct points $x \in S$. A straightforward relationship between these measures of complexity can be expressed through the following result.

\begin{lemma}[\cite{alon1997scale}]
\label{lem:complexitycompar}
    For any hyposthesis class $\mathcal{H} \subseteq [0, 1]^{\mathcal X}$ and $\gamma > 0$, we have
    \[
    \fatvf{\mathcal H}(\gamma) \le \fat{\mathcal H}(\gamma) \le \left(\frac{1}{\gamma} + 1\right)\fatvf{\mathcal H}(\gamma/2).
    \]
\end{lemma}

Given a training sample $S = ((x_i,y_i))_{i=1}^n$, recall that an ERM algorithm selects a hypothesis $\widehat{f}_{\operatorname{ERM}}$ as
\[
\widehat{f}_{\operatorname{ERM}} \in \underset{f \in \mc{H}}{\operatorname{argmin}} \ \frac{1}{n}\sum\limits_{i = 1}^n|f(x_i) - y_i|.
\]
To place the result of \Cref{thm:loreg} into better perspective, we present a more standard risk bound valid for any ERM classifier via a uniform convergence argument.
This bound is expressed in terms of $\fat{\mathcal H}(\gamma)$.
We define the covering numbers with respect to the empirical $L_1$ distance. Given the sample $x_1, \ldots, x_n$, we define, for any $f, g \in \mathcal H$, $\rho_n(f, g) = \frac{1}{n}\sum_{i = 1}^n|f(x_i) - g(x_i)|$. Let $\mathcal{N}_1(t, \mathcal H, n)$ denote the maximal (with respect to $x_1, \ldots, x_n$) covering number of $\mathcal H$ with respect to $\rho_n$ at scale $t$.
As we could not find an explicit reference in the literature,\footnote{There are closely related bounds in \cite[Section 21.4]{anthony1999neural}.} we provide a standard proof in \cref{ssec:ermperformance}.
To simplify the bound, we introduce a mild regularity assumption on the behavior of $\fat{\mc{H}} (\cdot)$ which allows us to use the powerful result of Rudelson and Vershynin \cite{rudelson2006combinatorics}.

\begin{restatable}{proposition}{ermperformance}
    \label{prop:ermperformance}
There are absolute $c_1, c_2 > 0$ such that the following holds.
Fix a hypothesis class $\mc{H} \subseteq [0, 1]^\mc{X}$ and margin parameter $\gamma \in (0,1)$. Assume that there is a constant $c_3 > 1$ such that for all $\theta > 0$, it holds that $\fat{\mc{H}} (c_3 \theta) \le \fat{\mc{H}} (\theta)/2$. For any $\delta \in (0, 1)$, every ERM algorithm that receives a training sample $S \sim P^n$ from a realizable distribution $P$ as input and outputs a hypothesis $\widehat{f}_{\operatorname{ERM}}$ satisfies
    \begin{equation*}
        \E_{(X,Y) \thicksim P} \lsrs{\abs{\widehat{f}_{\operatorname{ERM}}(X) - Y}} \leq c_1\lprp{\gamma + \frac{\fat{\mc{H}} (c_2\gamma)}{n} + \frac{1}{n}\log\left(\frac{1}{\delta}\right)},
    \end{equation*}
with probability at least $1 - \delta$ over the randomness of $S$.
\end{restatable}

\cref{lem:complexitycompar} demonstrates that the bound of \cref{thm:loreg} is never worse than the bound of \cref{prop:ermperformance}, up to constant multiplicative factors. However, the difference between the corresponding complexity measures can be quite significant, as shown by the following example which pertains to the classic problem of isotonic regression.  Our example also illustrates the result from \cite{rudelson2006combinatorics}, emphasizing the equivalence between the logarithm of covering numbers and the $P_{\gamma}$-dimension for general hypothesis classes.

\begin{example}[Isotonic regression with absolute loss]
\label{ex:monotonic}
Let $\mathcal H \subseteq [0, 1]^{[0, 1]}$ be a set of all monotonic functions. Then, for any $\gamma, t \in (0, 1]$,
\[
\fatvf{\mathcal H}(\gamma) \le 2, \quad \textrm{but} \quad \fat{\mathcal H}(\gamma) = \Theta\left(\frac{1}{\gamma}\right),\quad \textrm{and} \quad \log\mathcal{N}_1(t, \mathcal H, n) = \Theta\left(\frac{1}{t}\right).
\]
Specifically, in this case, the upper bound of \Cref{thm:loreg} scales as $O\left(\frac{1}{n}\log\left(\frac{1}{\delta}\right)\right)$, whereas the covering-number-based upper bound of \Cref{prop:ermperformance} scales as $O\left(\frac{1}{\sqrt{n}} + \frac{1}{n}\log\left(\frac{1}{\delta}\right)\right)$.
\end{example}

The bounds in this example are based on several existing results. The upper bound for $\fat{\mathcal H}(\gamma)$ follows from \cref{lem:complexitycompar}, while the lower bound can be directly constructed from the definition. The upper bound for $\log\mathcal{N}_1(t, \mathcal H, n)$ follows from \cite[Theorem 2.7.5]{van1996weak}, while the lower bound is implied by \cite[Theorem 12.10]{anthony1999neural}. 

One might think that employing a version of \cref{prop:ermperformance} with localized covering numbers could improve the ERM convergence rates in \cref{ex:monotonic}. 
Unfortunately, this does not help because global and local covering numbers have the same behavior for non-parametric classes \cite{yang1999information, massart2006risk}.
Finally, we mention that the recent work of Kot\l{}owski, Koolen, and Malek \cite{kotlowski2017random} explores in-expectation risk bounds for the class considered in \cref{ex:monotonic} in a similar learning setting.

\section{Related work}
\label{sec:related_work}
\paragraph{The optimal sample complexity of binary classification.}
Determining the optimal sample complexity, or equivalently the optimal risk bound,  of realizable (noiseless) binary classification is a classic problem in statistical learning theory. The early work of Vapnik and Chervonenkis \cite{vapnik1964class} introduced what we call the PAC model in this paper (a terminology due to Valiant \cite{valiant1984theory}, who also studied these questions from a computational perspective), providing matching upper and lower risk bounds for finite hypothesis classes $\mathcal{H}$. Subsequently, Vapnik and Chervonenkis presented the first upper bounds on the error rate for an infinite class, specifically the class of half-spaces in $\mathbb{R}^p$ \cite{vapnik1968algorithms}. Their next groundbreaking work \cite{vapnik1968uniform} generalized the shattering properties of half-spaces to general classes and introduced the complexity measure now known as the VC dimension. The monograph \cite{vapnik74theory} and the paper \cite{blumer1989learnability} provide the canonical risk bound applicable to any classifier consistent with the training sample:
\begin{equation}
\label{eq:ermupperbound}
O\left(\frac{d}{n}\log\left(\frac{n}{d}\right) + \frac{1}{n}\log\left(\frac{1}{\delta}\right)\right).
\end{equation}
This convergence rate is known to be unimprovable for certain sample consistent classification rules and hypothesis classes \cite{haussler1994predicting, auer2007new, simon2015almost, hanneke2016refined}.
However, the existing lower bound \cite{EhrenfeuchtHKV89} that applies to any predictor
\begin{equation}
\label{eq:lowerboundpac}
\Omega\left(\frac{d}{n} + \frac{1}{n}\log\left(\frac{1}{\delta}\right)\right)
\end{equation}
did not match the risk bound \eqref{eq:ermupperbound} by a multiplicative logarithmic factor.
The pivotal work of Haussler, Littlestone, and Warmuth \cite{haussler1994predicting} (originally presented as \cite{Haussler1988}) took a significant step towards proving an upper bound that matches the rate of \eqref{eq:lowerboundpac}. They introduced the one-inclusion graph algorithm and presented a simple confidence boosting strategy that occasionally provides a better risk bound
\begin{equation}
\label{eq:oigwhp}
O\left(\frac{d}{n}\log\left(\frac{1}{\delta}\right)\right),
\end{equation}
raising the question of what the optimal sample complexity of PAC learning is, and whether the one-inclusion graph algorithm is itself optimal \cite{warmuth2004optimal}. Since then, numerous authors have provided bounds matching the rate of the lower bound \eqref{eq:lowerboundpac} for specific classes \cite{blumer1989learning, Littlestone89, auer1997learning, auer2007new, darnstadt2015optimal, hanneke2016refined} and/or under additional assumptions on the distribution of the data \cite{long2003upper, bshouty2009using, balcan2013active, hanneke2016refined}. The problem of obtaining a rate matching \eqref{eq:lowerboundpac} up to constant factors was ultimately resolved by Hanneke \cite{hanneke2016optimal}, who improved upon an earlier result by Simon \cite{simon2015almost} which contained a slowly growing additional multiplicative term.

\paragraph{Leave-one-out/exchangeability arguments.}
The leave-one-out analysis is a standard approach to control the \emph{expected} risk of classification rules. Some early applications of this method can be traced back to the monograph of Vapnik and Chervonenkis \cite{vapnik74theory}, where they utilized it to derive optimal risk bounds for both the Perceptron algorithm and the hard margin SVM. However, these results were valid only in-expectation, prompting Vapnik and Chervonenkis \cite[Chapter VI, Section 7]{vapnik74theory} to ask whether the same could be achieved with high probability for some practically relevant algorithms. Lunts and Brailovsky \cite{lunts1967evaluation} showed that this is not possible in general by presenting an example where the leave-one-out argument failed to yield a low variance risk bound, thus excluding the possibility of attaining high-probability upper bounds without additional assumptions. The leave-one-out argument is also central to the analysis of the one-inclusion graph algorithm introduced in \cite{haussler1994predicting}. Warmuth \cite{warmuth2004optimal} conjectured that this algorithm consistently provides optimal PAC upper bounds. However, it was recently demonstrated in \cite{adenali2022oneinclusion} that, without additional assumptions, the one-inclusion graph algorithm is unable to achieve the optimal PAC bound. Lastly, we draw attention to the recent analysis of \emph{stable compression schemes} in \cite{Zhivotovskiy2017optimal, bousquet2020proper, hanneke2021stable}, which reveals that a leave-one-out argument can indeed result in sharp high probability bounds when stability assumptions are incorporated.

\paragraph{Confidence boosting and online to batch conversions.}
We first discuss the \emph{confidence boosting} approach, a natural method for deriving high-probability risk bounds from in-expectation ones. This technique involves running the in-expectation optimal algorithm on $O\left(\log(1/\delta)\right)$ independent partitions of the original sample and selecting the best performer based on an independent hold-out sample. Confidence boosting was employed by Haussler, Littlestone, and Warmuth \cite{haussler1994predicting}, however this approach only leads to a bound with a multiplicative $\log(1/\delta)$ as in \eqref{eq:oigwhp}.

A more pertinent approach is the online-to-batch conversion, with its earliest relevant applications tracing back to the works \cite{Aizerman1965, vapnik74theory}. Littlestone \cite{Littlestone89} demonstrated that if a conservative online algorithm (an algorithm that only changes its state when a point is misclassified) makes a finite number $M$ of mistakes on any sample, it can be converted into a classifier in a statistical setting with a risk bound
\begin{equation}
\label{eq:littlstoneoig}
O\left(\frac{M}{n} + \frac{1}{n}\log\left(\frac{1}{\delta}\right)\right).
\end{equation}
Despite providing a correct additive term containing $\log(1/\delta)$, Littlestone's bound has two major limitations, both addressed in this paper. The first limitation is the requirement of a finite number of mistakes, characterized by the so-called \emph{Littlestone dimension} \cite{littlestone1988learning}, which is infinite for most practically relevant hypothesis classes. Haussler, Littlestone, and Warmuth \cite{haussler1994predicting} attempted to address this issue by considering the \emph{expected} number of mistakes made by one-inclusion graph algorithms run on the sample in an online manner, resulting in an expected number of mistakes scaling as $O(d\log n)$. The second limitation of the bound \eqref{eq:littlstoneoig} is the necessity for the algorithm to be conservative. Although any algorithm achieving a deterministic mistake bound can be converted to a conservative one \cite{floyd1995sample}, this is not known to hold for algorithms achieving a mistake bound \emph{in-expectation} as is the case in our work.

A pertinent reference is the work of Wu, Heidari, Grama, and Szpankowski \cite{wu2022expected}.
They proved, via a reverse martingale argument, that \emph{high-probability online bounds} analogous to the in-expectation bound of Haussler et al. \cite{haussler1994predicting} can be attained, showing that $O(d\log n + \log(1/\delta))$ mistakes are made with probability at least $1 - \delta$.
Another observation stems from \cite{haussler1994predicting}, who noted that when the one-inclusion graph algorithm is executed in an online fashion, and its error is measured solely on the second half of the sample --- referred to as the \emph{suffix} --- the expected number of mistakes diminishes to $O(d)$. 
This idea of utilizing the suffix of the sample is quite general and has led to recent sharp high-probability bounds for stochastic gradient descent \cite{harvey2019tight}.

\paragraph{Multiclass classification and partial hypotheses.}

Multiclass classification is a natural extension of the binary case. The works \cite{natarajan1988two, natarajan1989learning} extended the PAC model to multiclass problems and introduced an analog of the VC dimension, now commonly referred to as the Natarajan dimension. Latter work asked whether this dimension characterizes learnability and highlighted the insufficiency of standard uniform convergence techniques for answering this question. Subsequent works exploring the Natarajan dimension include \cite{haussler1995generalization, bendavid95}. 
While the Natarajan dimension enables the uniform convergence property with a finite number of classes, ERMs may fail to learn the class when the number of classes is unbounded \cite{daniely2015multiclass}.
The authors of \cite{rubinstein2009shifting} adapted the one-inclusion-graph algorithm to multiclass classification, while Daniely and Shalev-Shwartz \cite{daniely2014optimal} characterized the expected optimal risk of multiclass problems by the one-inclusion hypergraph density. Recent results and a detailed survey on complexity measures in multiclass classification can be found in \cite{brukhim2022characterization}.

A partial hypothesis class consists of $\{0, 1, \star\}$-valued hypotheses. A key distinction from multiclass classification is the unique treatment of $\star$-values, which do not affect the notion of shattering or corresponding complexity measures. This idea traces back to the work of Haussler and Long \cite{haussler1995generalization} and is further explored in the work of Bartlett and Long \cite{bartlett1998prediction}. 
The authors in \cite{bartlett1998prediction} recognized that uniform convergence should not be relied upon for partial hypotheses and adapted the one-inclusion graph algorithm accordingly. Subsequent work by Long \cite{long2001agnostic} extended the analysis of partial hypotheses to the agnostic setting and derived corresponding risk bounds. The recent resurgence of interest in this topic has been sparked by the work of Alon, Hanneke, Holzman, and Moran~\cite{alon2022theory}. Their contributions include determining sharper bounds for both the realizable and agnostic cases, while also extending these ideas to online learning.

\section{Proof of main result}
\label{sec:proof_main_result}

In this section, we prove the main technical result of the paper, \cref{thm:main}. Recall that we aim to bound the average risk
\begin{equation}
    \label{eq:avg_err_proof_overview}
    \frac{\splita }{\splitb n} \cdot \sum_{t = n / \splita}^{n-1} \risk{\widehat{f}(\cdot; \Slt{})}{P}
\end{equation}
of a sequence of predictors constructed from the data with high probability.
As a first step, we analyze the empirical quantity
\begin{equation}
    \label{eq:mistake_bound_proof_overview}
    \frac{\splita }{\splitb n} \cdot \sum_{t = n / \splita}^{n-1} \ell (\wh{f}(X_{t + 1}; \Slt{}), Y_{t + 1}).
\end{equation}
Intuitively,  \eqref{eq:mistake_bound_proof_overview} approximates \eqref{eq:avg_err_proof_overview} since the test point $X_{t + 1}$ is independent of the training sample $\Slt{}$ used to construct the predictor. That is, $\ell (\wh{f}(X_{t + 1}; \Slt{}), Y_{t + 1})$ is an unbiased estimator of $\risk{\widehat{f}(\cdot; \Slt{})}{P}$. We then convert the in-expectation bound to a PAC bound using a martingale concentration argument in \cref{lem:forward}. 

The main technical challenge is to bound \eqref{eq:mistake_bound_proof_overview} with high probability. This step makes crucial use of the fact that the predictors under consideration are permutation invariant as captured by \cref{as:perm_inv}. Crucially, our analysis conditions on the values and labels of the inputs $\lprp{(X_i, Y_i)}_{i = 1}^n$ but not on their \emph{permutation} in the input sequence. The key insight in the proof is the following conditional independence structure. For any $i < j$, $\ell (\wh{f}(X_{i + 1}; \Slt[i]), Y_{i + 1})$ has small expected value even when conditioned on $\ell (\wh{f}(X_{j + 1}; \Slt[j]), Y_{j + 1})$. This is due to the fact that $\ell (\wh{f}(x_{j + 1}; \Slt[j]), y_{j + 1})$ is determined by the \emph{set} of training samples $\Slt[j]$ and $(X_{j+1},Y_{j+1})$ while $\ell (\wh{f}(x_{i + 1}; \Slt[i]), y_{i + 1})$ depends on the \emph{set} of training samples $\Slt[i]$ and $(X_{i+1},Y_{i+1})$. Due to the exchangability of $\Slt[i]$ even given the set $\Slt[j]$, we can relate $\ell (\wh{f}(X_{j + 1}; \Slt[j]), Y_{j + 1})$ has small expected value from \cref{as:bdd_loo}. 
This intuition is formalized in \cref{lem:reverse} by constructing an appropriate filtration and using a simple martingale concentration bound. 

To prove our results, we need the following martingale Chernoff bounds. 
We include a proof of \cref{lem:chernoff-ish} in \cref{ssec:app_martingale} for completeness.

\begin{restatable}{lemma}{chernoffish}\label{lem:chernoff-ish}
Let $W_1, \ldots, W_T$ be a stochastic process adapted to the filtration $(\mathcal{F}_i)_{i \le T}$. 
Suppose that $0 \leq W_t \leq 1$ almost surely. 
Then, for any $ \delta, \lambda,\eta \in (0,1)$, we have
\begin{align}
    &\Pr\left[\sum\limits_{t = 1}^T W_t \geq \frac{e^\lambda-1}{\lambda}\sum\limits_{t = 1}^T \E[W_t| \mathcal{F}_{t-1}] + \frac{\log(1/\delta)}{\lambda} \right] \leq \delta, \label{eq:chernoffish1}
 \end{align}
and
 \begin{align}
    &\Pr\left[\sum\limits_{t = 1}^T \E[W_t| \mathcal{F}_{t-1}] \geq \frac{\eta e^\eta}{e^\eta-1}\sum\limits_{t = 1}^T W_t + \frac{e^\eta\log(1/\delta)}{e^\eta-1} \right] \leq \delta.\label{eq:chernoffish2}
\end{align}
\end{restatable}
We first formally look at the argument connecting the empirical loss to the risk. 
\begin{lemma}[Forward martingale risk bound]\label{lem:forward}
    Fix a realizable distribution $P$ and let $ \ell $ be a loss bounded by $1$. 
    Let $\widehat{f} : \mc{X} \times \mc{U} \to \mc{Y}$ be a predictor.
    Given a training sample $S =((X_1,Y_1), \dots, (X_T,Y_T)) \sim P^{T}$, for any $\delta, \eta \in (0,1)$, we have that
    \begin{align*}
        \sum_{t=T/\splita}^{T-1} \risk{\widehat{f}  \left(  \cdot ; S_{ \leq t } \right)     }{P} \leq \frac{\eta e^\eta}{e^\eta -1} \sum_{t=T/\splita}^{T-1}   \ell\left(  \widehat{f}  \left(  X_{t+1} ; S_{ \leq t  } \right)   ,Y_{t+1}\right) +   \frac{e^\eta\log(1/\delta)}{e^\eta-1},
    \end{align*}
    with probability at least $1-\delta$ over the randomness of $S$.
    \end{lemma}
\begin{proof}
    For $t=1, \dots, T$ define the usual filtration $(\mc{F}_t)_{t=1}^T$ where $\mc{F}_{t} = \sigma(X_1, \dots, X_t)$.
    Here, $ \sigma\left( \cdot \right)$ denotes the smallest $\sigma$-algebra with respect to which the arguments are measurable.
    With slight abuse of notation, we write $\widehat{f}_{t}(\cdot) = \widehat{f}  \left(  \cdot ; \Slt \right)$. 
    By definition $\risk{\widehat{f}_{t}}{P} = \E[\ell(\widehat{f}_{t}(X_{t+1}),Y_{t+1})\mid \mc{F}_{t}]$, so applying \cref{lem:chernoff-ish} to the random variables 
    \[
    \ell(\widehat{f}_{T/\splita}(X_{T/\splita+1}),Y_{T/\splita+1}),\dots, \ell(\widehat{f}_{T-1}(X_T),Y_t)
    \]
    gives us that
    \[
    \sum_{t=T/\splita}^{T-1} \risk{\widehat{f}_{t}}{P} \leq \frac{\eta  e^\eta }{e^\eta -1} \sum_{t=T/\splita}^{T-1}   \ell(\widehat{f}_{t}(X_{t+1}),Y_{t+1}) +   \frac{e^\eta \log(1/\delta)}{e^\eta -1},
    \]
    with probability at least $1-\delta$ over the randomness of $S$.
\end{proof}
As noted earlier, \cref{lem:forward} tells us that to bound the average risk of the predictors $\widehat{f}  \left(  \cdot , S_{ \leq t  } \right) $, it suffices to bound the loss when evaluated on the points $X_t$.
In order to do this, we use a reverse martingale argument. These arguments are standard in the study of empirical processes since the work of Pollard \cite{Pollard} and have found recent applications in sequential estimation and uncertainty quantification. The results of Manole and Ramdas \cite[Proposition 5]{manole2023} make a connection between leave-one-out arguments and reverse martingales. Wu, Heidari, Grama, and Szpankowski~\cite{wu2022expected} used similar arguments specifically for the study of sequential covering numbers using the one-inclusion graph algorithm. The proof below could shortened by translating it to the language of reverse martingales and the exchangeable filtration, but we chose an elementary presentation in terms of random permutations for clarity.
\begin{lemma}[Reverse martingale bound]\label{lem:reverse}
Fix a realizable distribution $P$. Let $\widehat{f} :  \mc{X} \times \mc{U} \to \mc{Y}$ be a predictor satisfying \cref{as:perm_inv,as:bdd_loo} with leave-one-out error $M_{n}$.
Given a training sample $S = (Z_1 , \dots, Z_T ) =  ((X_1,Y_1), \dots, (X_T,Y_T))\sim P^{T}$, for any $\delta, \lambda \in (0,1)$, we have that
\begin{align*}
\sum_{t =T/\splita }^{T-1} \ell\left(  \widehat{f}  \left(  X_{t+1} ; \Slt \right)   ,Y_{t+1}\right) \leq \left(\ln(\splita)+\frac{1}{2}\right)\frac{e^\lambda - 1}{\lambda}M_{T} + \frac{\log(1/\delta)}{\lambda},
\end{align*}
with probability at least $1-\delta$ over the randomness of $S$. 
\end{lemma}
\begin{proof}  
    Let $I_t = \ell(\widehat{f}(X_{T-t+1};\Slt[T-t]),Y_{T-t+1})$. 
    Note that $I_t$ is a function of $Z_1, \dots , Z_{T-t+1} $ and when required we will make this dependence explicit by writing $I_t( Z_1 \dots , Z_{T-t+1} )$. 
    Further, observe that since $Z_1 \dots , Z_{T-t+1} $ are exchangeable random variables, we have that the joint distribution of $ Z_{\pi(1)} \dots , Z_{\pi(T)}$ does not depend on the permutation $\pi$. 
    Consider a {uniformly} random permutation $\pi$ independent of $Z_1, \dots , Z_T$. 
    That is, throughout the proof we may assume that the values $Z_1, \dots , Z_T$ are fixed.
    Consider the filtration 
    \begin{align*}
        \mathcal{F}_t = \sigma\left( \left\{  Z_{\pi(1)} , \dots ,  Z_{ \pi(T-t) }  \right\} ,  Z_{\pi(T-t+1 )}, \dots ,  Z_{ \pi(T)}  \right).
    \end{align*} 
    Here, $\left\{  Z_{\pi(1)} , \dots ,  Z_{ \pi(T-t) }  \right\} $ denotes the set of values taken by $  Z_{\pi(1)} , \dots ,  Z_{ \pi(T-t) }   $ (excluding their order) and $ \sigma\left( \cdot \right)$ denotes the smallest $\sigma$-algebra with respect to which the arguments are measurable.
    Let $S_\pi = ((X_{\pi (1)}, Y_{\pi (1)}),\dots, (X_{\pi (n)}, Y_{\pi (n)}))$ denote the permutation of the training sample with respect to $\pi$. Define 
    \begin{align*}
        I_t' =  I_t(   Z_{\pi(1)} , \dots ,  Z_{ \pi(T-t+1 ) }   ) =  \ell(\widehat{f}(X_{\pi(T-t+1)}; (S_{\pi})_{\le T-t}) , Y_{ \pi(T-t+1) }). 
    \end{align*}
     Note that $I'_t$ is adapted to the filtration $(\mathcal{F}_t)_{t = 1}^T$ since, by \cref{as:perm_inv},
     $ \widehat{f} $
     is permutation invariant on the training sample $(S_{ \pi } )_{\leq T-t }$. 
     Applying \eqref{eq:chernoffish1} gives us that with probability at least $1-\delta$, we have
    \begin{align}
        \sum_{t=1}^{3T/4} I'_t \leq \frac{e^{\lambda} - 1 }{\lambda} \sum_{t= 1}^{3T/4} \mathbb{E} \left[  I'_t | \mathcal{F}_{t-1}  \right] + \frac{ \log(1/ \delta) }{\lambda}. \label{eq:mistaketail} 
    \end{align}
   To use this bound, we need to control
      \begin{align}
        \mathbb{E} \left[  I'_t | \mathcal{F}_{t-1}  \right] &= \mathbb{E} \left[  I( Z_{\pi(1)} , \dots , Z_{ \pi(T-t+1) }   ) | \mathcal{F}_{t-1} \right]   \label{eq:expec}. 
    \end{align}
    Notice that at time $t-1$ (with respect to the filtration $(\mathcal{F}_t)_{t = 1}^T$) the random indices $\pi(1), \dots, \pi(T-t+1)$ correspond to a uniformly random permutation of the elements of the set $[T] \setminus \{\pi (i)\}_{i = T-t+2}^{T}$. A standard computation relates averaging over uniformly random permutations to the leave-one-out error from which we can conclude together with \eqref{eq:expec} and \cref{as:bdd_loo} the bound
    \[
    \mathbb{E} \left[  I'_t | \mathcal{F}_{t-1}  \right] \leq \frac{M_{T-t+1}}{T-t+1}.
    \] 
    Plugging the above into \eqref{eq:mistaketail} along with the exchangeability of $Z_1, \dots , Z_T$ gives us 
    \begin{align*}
        \sum_{t =T/\splita }^{T-1} \ell\left(\widehat{f}(X_{t+1} ; \Slt),Y_{t+1}\right) &\leq \frac{e^\lambda -1}{\lambda} \sum_{t=1}^{3T/4} \frac{M_{T-t+1}}{T-t+1} + \frac{\log(1/\delta)}{\lambda} \\ 
        &\le \frac{e^\lambda -1}{\lambda} M_{T} \sum_{t=T/4 + 1}^T \frac{1}{i } + \frac{\log(1/\delta)}{\lambda} \numberthis \label{eq:mono} \\ 
        & \leq \frac{e^\lambda -1}{\lambda} M_{T} \left( \ln 4 + \frac{1}{2} \right)   + \frac{\log(1/\delta)}{\lambda}.   \numberthis \label{eq:harmonic}
     \end{align*} 
     The inequality \eqref{eq:mono} follows from the monotonicity of the  function $M_t$.  Finally,
     \eqref{eq:harmonic} follows from a simple approximation to the harmonic series considered.
 \end{proof}
We are now ready to prove our main result by piecing together the two results above.
We restate the theorem for convenience.
\main*
\begin{proof}[Proof of \cref{thm:main}]
Using \cref{lem:forward,lem:reverse} with confidence parameter $\delta/2$ together with a union bound implies, with probability at least $1-\delta$ over the randomness of $S$, that
\begin{align*}
\frac{\splita}{\splitb n}\sum_{t=n/\splita}^{n-1} \risk{\widehat{f}(\cdot ; \Slt)}{P} &\leq \frac{\splita}{\splitb n}\left(\frac{\eta e^{\eta}}{e^{\eta}-1} \left((\ln(\splita)+1/2)\frac{e^\lambda - 1}{\lambda}M_{n} + \frac{\log(1/\delta) }{\lambda}\right) + \frac{e^{\eta}\log(2/\delta)}{e^{\eta}-1} \right).
\end{align*}
Setting the parameters\footnote{The choices of the parameters have been optimized to get a small constant.} to $\lambda = 0.82$ and $\eta = 0.78$ followed by a simple calculation gives us the upper bound
\[
\frac{\splita}{\splitb n}\sum_{t=n/\splita}^{n-1} \risk{\widehat{f}(\cdot ; \Slt)}{P} \leq \cishaq\left( \frac{M_n}{n} + \frac{1}{n}\log \left(\frac{2}{\delta}\right)  \right).
\]
This completes the proof.
\end{proof}

\paragraph{Acknowledgments.} The authors gratefully acknowledge Peter Bartlett and Phil Long for useful discussions on the bounded regression setting. IA and AS would like to thank Prasad Raghavendra for offering a course where this work began as a final project. 

\bibliographystyle{alpha}
\bibliography{refs}

\appendix 

\section{Proofs from \cref{ssec:app_classification}: Classification}\label{app:oig} 

In this section we provide relevant details on the one-inclusion hypergraph algorithm \cite{haussler1994predicting,rubinstein2009shifting}. 
We also provide proofs for the results claimed in \cref{ssec:app_classification}.

We begin with some preliminary definitions for hypergraphs.
We define an \emph{orientation} of hypergraph $G = (V,E)$ to be a function $\sigma_{G} : E \to V$ such that for any $e \in E$, $\sigma_{G}(e) \in e$.
Given a hypergraph $G = (V,E)$ and orientation $\sigma_G$, define the \emph{out-degree} of a vertex $v \in V$ to be
\[
\out(v;\sigma_G) = |\{e \in E : v \in e, \sigma_{G}(e) \not= v\}|.
\]
For a hypergraph $G$ oriented using $\sigma_G$, the \emph{max out-degree} of $G$ is defined to be $\out(\sigma_G) = \max_{v \in V} \out(v;\sigma_G)$.

We are now ready to describe the one-inclusion hypergraph algorithm as pseudocode in~\cref{alg:OIG}.

\begin{algorithm}[H]
\caption{One-inclusion hypergraph algorithm.}
\label{alg:OIG}
\textbf{Inputs:} Training sample $S$.\\
\textbf{Output:} Hypothesis $\Oig{S} : \mc{X} \to \mc{Y}$.
\vspace{1em}

For any point $x \in \mc{X}$, $\Oig{S}$ predicts as follows:\hspace{-5em}
\vspace{1em}
\begin{algorithmic}[1] 
    \STATE Build the hypergraph $G(\mc{H}|_{\uniS \cup \{x\}})$ and pick an orientation $\sigma_{\uniS \cup \{x\}}$ \emph{minimizing} $\out(\sigma_{\uniS \cup \{x\}})$.\hspace{-1em}
    \vspace{0.5em}
    \STATE If there is a unique label $y$ for $x$ consistent with $S$ and $\mc{H}$, predict $y$.\hspace{-1em}
    \vspace{0.5em}
    \STATE Otherwise, let $e$ be the edge in $G(\mc{H}|_{\uniS \cup \{x\}})$ with hypotheses consistent with $S$ but not on $x$.
    \vspace{0.5em}
    \STATE Predict according to $ \sigma_{\uniS \cup\{x\}} (e)$, i.e.,\ the hypothesis pointed to in the orientation of $e$.
\end{algorithmic}
\end{algorithm}
The one-inclusion (hyper)graph algorithm \cite{haussler1994predicting,rubinstein2009shifting} was originally defined for the \emph{transductive} model of learning, i.e., it predicts a label for a single test point $x$.
The algorithm as presented above is the standard extension that defines a predictor for the entire domain $\mc{X}$.
When predicting on a test point $x$, the algorithm minimizes $\out(G(\mc{H}|_{\uniS \cup \{x\}}))$ which upper bounds its leave-one-out error (\cref{lem:oig_loo}).
Before we explore the leave-one-out error of this algorithm, we connect $\out(G(\mc{H}|_{\uniS \cup \{x\}}))$ to $\dens$. 

Every hypergraph $G$ has an orientation $\sigma_G$ for which the max out-degree $\out(\sigma_G)$ is at most $ \lceil \mu(G) \rceil$.
This is a standard fact for graphs (e.g., see~\cite[Lemma 3.1]{alon1992colorings}, \cite[Theorem 2.2]{haussler1994predicting}).
We expect the generalized statement for hypergraphs to also be known, but since we could not find an explicit reference, we include a proof based on a maximum flow argument generalizing a known proof for the graph case \cite{haussler1994predicting}.
\begin{lemma}\label{lem:orientatingdensity}
For any hypergraph $G = (V,E)$ and any integer $d \ge \mu(G)$, there is an orientation $\sigma_G$ such that the max out-degree satisfies $\out(\sigma_G) \le d$.
\end{lemma}
\begin{proof}
We build an auxilary graph that contains a source $s$, a sink $t$, and two vertex sets $V_1$ and $V_2$ between $s$ and $t$.
$V_1$ contains a vertex for each edge in $E$, i.e.,\ $V_1 = E$. 
$V_2$ is equal to the vertex set of $G$, i.e.,\ $V_2 = V$.
There is a weighted and directed edge from $s$ to every $e \in V_1$.
For each $e \in V_1$, the weight of the edge to $e$ from $s$ is equal to $|e|-1$.
For every $e \in V_1$, there is a weighted directed edge from $e$ to every $v \in V_1$ that participates in the hyperedge $e$ in $G$.
Each such edge has weight $1$.
Finally, there is an edge from every $v \in V_2$ to $t$, each with weight $d$.
We will prove that the maximum flow of this graph from $s$ to $t$ is equal to $m \coloneqq \sum_{e \in V_1} |e|-1$.

Since all the weights are integers, there exists an integral maximum flow.
Consider such a flow.
There is a simple way to obtain the desired orientation of $G$ from this flow.
Fix an $e \in V_1$.
Since $e$ gets $|e|-1$ units of flow, $|e|-1$ of its $|e|$ neighbours gets a unit of flow.
Thus, there is a single neighbour $v \in V_2$ of $e$ that does not get a unit of flow, and we orient the edge $e$ in the original hypergraph $G$ towards this $v$.
Repeat this for every $e \in V_1$.
Every unit of flow some $v \in V_2$ recieves corresponds to an edge in $G$ that it is not selected to be the head of. Since the capacity of the edge out of every $v \in V_2$ to $t$ is $ d $, we can conclude that the orientation satisfies that the out-degree of every $v \in V$ is at most $d$.

We now prove the maximum flow is $m$.
It is easy to see the flow is at most $m$ since the sum of the edge weights from $s$ to $V_1$ is equal to $m$.
It remains to prove the maximum flow is at least $m$.
Recall that the maximum flow is equal to the minimum cut (set of edges) that separates $s$ from $t$, so it suffices to prove a lower bound for the minimum cut.

The following is a useful observation we will make use of.
Notice that every path from $s$ to $t$ consists of 3 edges passing through some $e \in V_1$ and $v \in V_2$.
Fix one such path.
Since we consider a minimum cut, if the first edge (connecting $s$ to $e$) is cut, there is no need to cut the edge between $e$ and $v$ since this would only increase the cut value. 
Similarly, if the third edge is cut (connecting $v$ to $t$) there is no need to cut the edge between $e$ and $v$.

Let $U \subseteq V_1$ be the set of vertices whose edge from $s$ \emph{are not} cut.
Let $W \subseteq V_2$ be the set of vertices whose edges to $t$ \emph{are} cut.
Fix an $e \in U$.
There are $|e|$ paths from $e$ to $t$ and each of these must be cut in either the middle layer of the edges or the third layer of edges, but not both.
Let $W_{e,2} \subseteq V_2$ be the set of vertices for which the paths from $e$ to $t$ are cut in the second layer and let $W_{e,3} \subseteq V_2$ be the set of vertices for which the paths from $e$ to $t$ are cut in the third layer. Clearly $W_{e,3} \subseteq W$.
Since we consider the minimum cut, we can conclude that $W_{e,3} \cap W_{e,2} = \emptyset$ from our earlier observation. 

We now upper bound the sum of the weight of the edges from $s$ to $U$.
Write this sum as $a = \sum_{e \in U} |e| -1$ which can be written $\sum_{e \in U} |e \cap W_{e,3}|+|e \cap W_{e,2}| -1$.
Let $W' = \cup_{e \in U} W_{e,3}$ and notice that 
\begin{align*}
\sum_{e \in U} |e \cap W'| - 1 &\le \sum_{e \in V_1} |e \cap W'| - 1 \le \mu(G)|W'| \le d |W'|.
\end{align*}
Thus, $a \le d|W'| + \sum_{e \in V_1} |e\cap W_{e,2}|$, so the sum of the weights of edges in the first layer that \emph{are} cut is at least $m - a \ge m - d|W'| - \sum_{e \in V_1}|e\cap W_{e,2}|$.
Since the weight of every edge in the second layer is $1$, the sum of weights of edges in the second layer that are cut can be written as $\sum_{e \in U}|e\cap W_{e,2}|$ as a result of the earlier observation.
Finally, since every edge in the final layer has weight $d$, the total weight of the cut edges in the third layer is equal to $d|W|$.
Since $W' \subseteq W$ we can conclude that the total weight of cut edges is at least 
\[
m - d|W'| - \sum_{e \in U}|e\cap W_{e,2}| +  d|W| + \sum_{e \in U} |e \cap W_{e,2}| \ge m.
\]
Thus, the maximum flow is indeed $m$.
\end{proof}
We can now use \cref{lem:orientatingdensity} to prove a bound on the LOO error of the one-inclusion graph algorithm.
\begin{lemma}\label{lem:oig_loo}
Fix a hypothesis class $\mc{H} \subseteq \mathcal{Y}^{\mathcal{X}}$.
The leave-one-out error of the one-inclusion graph algorithm $\widehat{f}_{\operatorname{OIG}}: \mc{X} \times \mc{U} \to \mc{Y}$ on \emph{any} realizable training sample $S = ((x_1, y_1), \dots, (x_{n},y_{n}))$ satisfies
\[
    \sum_{i = 1}^{n} \mathbf{1}\left\{\Oig[x_i]{\Si} \not = y_i \right\} \leq \lceil \dens \rceil.
\]
\end{lemma}
\begin{proof}
Let $\sub = \{x_1, \dots, x_n\}$, $G(\mc{H}|_{\sub}) = (V,E)$ be the one-inclusion hypergraph built using $\sub$, and $\sigma_{\sub}$ be an orientation that minimizes $\out(G(\mc{H}|_{\sub}))$.
Furthermore, let $f^* \in V$ be the hypothesis that labels all points in $U$ correctly, i.e.,\ $f^*(i) = y_i$ for all $i\in [n]$.
Notice that \cref{alg:OIG} builds the same one-inclusion hypergraph and orients it the same way for every choice of held out $x_i$.
So, the number of times the predictor makes a mistake is equal to the number of times it does not select $f^*$.
We have
    \[
    \sum_{i = 1}^{n} \mathbf{1}\left\{\Oig[x_i]{\Si} \not = y_i \right\} = \sum_{i = 1}^{n} \mathbf{1}\left\{\sigma_{\sub}(e_{i,f^*}) \not = f^* \right\} = \out(f^* ; \sigma_{U}) \le \lceil \dens \rceil,
    \]
where the inequality follows from \cref{lem:orientatingdensity}.
\end{proof}
For a set of hypotheses $\{f_1, \dots, f_t\} \subseteq \mc{H}$, we define their plurality vote $ \maj(f_1, \dots, f_t)$ to be the function that, on input $x \in \mc{X}$, outputs the label $y \in \mc{Y}$ that is most common out of $f_1(x), \dots, f_t(x)$.
We now prove our main multiclass PAC bound.
We restate \cref{thm:risk_multiclass} below before its proof for ease of reading.
\riskmulticlass*
\begin{proof}
Given the full training sample $S$, let $(\Oig{\Slt})_{t = 1}^{n-1}$ be the one-inclusion hypergraph predictors trained on the sequence of growing training samples $(\Slt)_{t=1}^{n-1}$. Because these predictors are symmetric in their training sample and they have a leave-one-out error guarantee (\cref{lem:oig_loo}), we can apply \cref{thm:main} to the suffix of the sequence to get that 
\[
\frac{4}{3n}\sum_{t=n/4}^{n-1}\err{\Oig{\Slt}}{P} \le \cishaq\left(\frac{\lceil\dens\rceil}{n}+\frac{1}{n}\log\left(\frac{2}{\delta}\right)\right),
\]
with probability at least $1-\delta$ over the randomness of $S$.
Whenever this bound holds, the plurality vote of the one-inclusion hypergraph predictors in the suffix, $\widehat{f}(\cdot) = \maj\left(\Oig{\Slt[\frac{n}{4}]}, \dots, \Oig{\Slt[n-1]}\right)$, satisfies
\begin{align*}
    \err{\widehat{f}}{P} &= \E_{(X,Y) \sim P}[\mathbf{1}\{\widehat{f}(X) \not= Y\}] 
    \\&\le \E_{(X,Y)\sim P}\left[2\left(\frac{4}{3n}\sum_{t=n/4}^{n-1}\mathbf{1}\{\Oig[X]{\Slt} \not= Y\}\right)\right]\\
    &= 2\cdot\frac{4}{3n}\sum_{t=n/4}^{n-1}\err{\Oig{\Slt}}{P}
    \\
    &\le \twocishaq\left(\frac{\lceil\dens\rceil}{n}+\frac{1}{n}\log\left(\frac{2}{\delta}\right)\right).
\end{align*}
In the first inequality above, we used the fact that no more than half the predictors can be correct when the majority vote is wrong.
This completes the proof.
\end{proof}

\subsection{Partial hypothesis classes}\label{app:oig_partial}
As mentioned in \cref{ssec:app_classification}, we can apply main multiclass risk bound in partial hypothesis classification by slightly modifying the one-inclusion hypergraph algorithm. 
We expand on this further.
Fix a partial hypothesis class $\mc{H} \subseteq \{0,1,\star\}^{\mc{X}}$.
If we were to follow the definition of the one-inclusion hypergraph, given any subset of the domain $\sub \subseteq \mc{X}$, we would first compute the projection $\mc{H}|_{\sub}$ and build the appropriate set of hyperedges.
For partial hypothesis classes, we will slightly change the definition of the one-inclusion hypergraph by removing every vertex (hypothesis) in the projection $\mc{H}|_{\sub}$ that labels a training point $x \in \sub$ with the $\star$ label, and then build the set of edges in the usual way from this truncated set of vertices.
This is motivated by the fact that an optimal hypothesis $f^* \in \mc{H}$ would never produce a $\star$ label on points sampled from a realizable distribution $P$.
It is easy to see that this change makes our one-inclusion hypergraph a graph.
We can also redefine the one-inclusion hypergraph density $\dens$ of $\mc{H}$ after making this change.
It now follows immediately from the arguments in the binary classification setting (\cref{thm:hlw_density}) that the one-inclusion hypergraph density $\dens$ of $\mc{H}$ is at most $d$, the VC dimension of $\mc{H}$.

\section{Proofs for \cref{ssec:app_regression}: Bounded regression}
\label{sec:proofs_bdd_regression_app}

\subsection{Proof of \cref{prop:ermperformance}} \label{ssec:ermperformance}

\ermperformance*

\begin{proof} Without loss of generality, we assume that in the realizable case, it holds that $Y = f^*(X)$ for some $f^* \in \mathcal H$. First, we use a standard ratio-type bound on empirical processes.
Theorem 19.7 in \cite{anthony1999neural} implies that for any $t \ge 0$,
\begin{equation}
\label{eq:uniformbound}
\Pr\left(\sup\limits_{f \in \mathcal H} \E_{X \thicksim P} \lsrs{\abs{{f}(X) - f^*(X)}} - \frac{2}{n}\sum\limits_{i = 1}^n|f(X_i) - f^*(X_i)| \ge t \right) \le 2\mathcal{N}_{1}\left(\frac{t}{12}, \mathcal{H}^{\prime}, 2n\right)\exp\left(-\frac{2nt}{9}\right),
\end{equation}
where $\mathcal{H}^{\prime} = \{x \mapsto |f(x) - f^*(x)|: f \in \mathcal H\}$.
For any distribution $P^{\prime}$ we have \[
\E_{X \thicksim P^{\prime}}\lsrs{\abs{\abs{f(X) - f^*(X)} - \abs{g(X) - f^*(X)}}} \le \E_{X \thicksim P^{\prime}}\lsrs{\abs{f(X) - g(X)}}.
\]
Thus, it is sufficient to bound $\mathcal{N}_{1}\left(\frac{t}{12}, \mathcal H, 2n\right)$ in terms of the $P_{\gamma}$-dimension. The result of Theorem 1.3 in \cite{rudelson2006combinatorics} does exactly this.\footnote{The result of Theorem 1.3 is formulated for $L_2$ distance, but according to the discussion on page 607 in \cite{rudelson2006combinatorics}, it also holds for all $L_p$, $1 \le p < \infty$ distances with appropriate changes in absolute constants. Similar generalizations are made in \cite{mendelson2003entropy}.} In fact, under the regularity assumption in our statement, this result implies
\[
\log 2\mathcal{N}_{1}\left(\frac{t}{12}, \mathcal H, 2n\right) \le c_4\fat{\mc{H}}(c_2t),
\]
where $c_2, c_4 > 0$ are absolute constants. Observe that $\fat{\mc{H}} (c_2\gamma)$ takes only integer values and decreases when $\gamma$ increases. Thus, one may find $\gamma^*\in[0, 1]$ such that $\fat{\mc{H}} (c_2\gamma^*) \le \frac{\gamma^*n}{9c_4}\le  \fat{\mc{H}} (c_2\gamma^*) + 1$. Fix $t = \gamma^*$. Using the definition of ERM and \eqref{eq:uniformbound} we have that, with probability at least $1 - \exp(-n\gamma^*/9)$,
\[
\E_{X \thicksim P} \lsrs{\abs{\widehat{f}_{\operatorname{ERM}}(X) - f^*(X)}} \le \gamma^*.
\]
If $n\gamma^*/9 \ge \log(1/\delta)$, then the above inequality holds with probability at least $1 - \delta$. Otherwise, we can instead consider $t = \frac{9}{n}\log(1/\delta)$, which satisfies $t \ge \gamma^*$, and thus by monotonicity $c_4\fat{\mc{H}} (c_2 t) \le  c_4\fat{\mc{H}} (c_2\gamma^*) \le \frac{\gamma^*n}{9} \le \log(1/\delta)$. This implies that whenever $n\gamma^*/9 < \log(1/\delta)$,
we have, with probability at least $1 - \delta$,
\[
\E_{X \thicksim P} \lsrs{\abs{\widehat{f}_{\operatorname{ERM}}(X) - f^*(X)}} \le \frac{9}{n}\log\left(\frac{1}{\delta}\right).
\]
Combining these bounds, we show that, with probability at least $1 - \delta$,
\[
\E_{X \thicksim P} \lsrs{\abs{\widehat{f}_{\operatorname{ERM}}(X) - f^*(X)}}  \le \gamma^* + \frac{9}{n}\log\left(\frac{1}{\delta}\right).
\]
Now take any $\gamma \in [0, 1]$. If $\gamma > \gamma^*$, then we obviously have $\gamma^* \le \gamma + \frac{\fat{\mc{H}} (c_2\gamma)}{n}$ and the claim follows. Otherwise, if $\gamma \le \gamma^*$, then $\gamma^* \le \frac{9c_4(\fat{\mc{H}} (c_2\gamma^*) + 1)}{n} \le \frac{9c_4(\fat{\mc{H}} (c_2\gamma) + 1)}{n}$. The claim follows by adjusting the constants.
\end{proof}

\subsection{Proof of \cref{lem:oiglo}}
\label{ssec:proof_oiglo}

Before going into details, we provide some definitions.
An alternative and equivalent definition of $\fatvf{\mathcal H}(\gamma)$, suggested in \cite{bartlett1998prediction}, naturally leads to the notion of partial hypothesis classes and facilitates connections with the one-inclusion graph algorithm. We define the thresholding operator for a fixed margin $\gamma$ and threshold $\tau$ as follows:

\begin{equation*}
    \psi_{\gamma, \tau} (z) = 
        \begin{cases}
            0 & \text{if } z \leq \tau - \gamma, \\
            1 & \text{if } z \geq \tau + \gamma, \\
            \star & \text{otherwise}.
        \end{cases}
\end{equation*}

We now formally reintroduce the complexity measure from this perspective. 
\begin{definition}
\label{def:l1comp}
For $\gamma > 0$, a sequence $(x_i)_{i = 1}^d \subseteq \mc{X}$ is $V_\gamma$-shattered by a function class $\mc{H}$ if there exists some $\tau \in [0, 1]$ such that
\begin{equation*}
\{0, 1\}^d \subseteq \lbrb{(\psi_{\gamma, \tau} (f(x_1)), \dots , \psi_{\gamma, \tau} (f(x_d))): f \in \mc{H}}.
\end{equation*}
We define $\fatvf{\mc{H}}(\gamma)$ as the length of the largest sequence $V_\gamma$-shattered shattered by $\mc{F}$.
\end{definition}

Contrasting with the partial hypothesis setup, we choose a specific loss function over the \emph{discrete} predictions $\{0, 1, \star\}$ that is insensitive to misclassification on $\star$. We define
\begin{equation}
\label{eq:bloss}
\ell_b (\widehat{y}, y) = \bm{1} \lbrb{\widehat{y} \neq y \text{ and } y \neq \star}.
\end{equation}
 We can now define the VC dimension of the partial hypothesis class $\mc{G} \subseteq \lbrb{0, 1, \star}^{\mc{X}}$ the same way as in \cref{ssec:app_classification}.
We recall the properties of a one-inclusion predictor from \cite{bartlett1998prediction}, which extends the classical one-inclusion graph algorithm of Haussler, Littlestone, and Warmuth \cite{haussler1994predicting}.

\begin{lemma}[\cite{bartlett1998prediction}]
\label{lem:oiglob}
Fix a hypothesis class $\mc{H} \subseteq \{0,1, \star \}^{\mc{X}}$ with VC dimension $d$. Let $f^* \in \mc{G}$, and consider the loss function \eqref{eq:bloss}. There exists a predictor $\widehat{f}$ such that, for any realizable sample $S = {(x_i, f^*(x_i))}_{i = 1}^n$, it satisfies \cref{as:perm_inv,as:bdd_loo} with $M_n \leq d$.
\end{lemma}

Instead of delving into the details, we provide an informal explanation of how the predictor $\widehat{f}$ from \cref{lem:oiglob} operates. Upon observing a realizable sample of size $n - 1$, we focus solely on the indices where $f^*(x_i) \neq \star$. We project our class onto these indices only. Consequently, the problem is transformed into a $\{0, 1\}$-valued classification problem, and we predict the label of the $n$-th observation using the standard one-inclusion graph prediction strategy of Haussler, Littlestone, and Warmuth \cite{haussler1994predicting}. Note that our loss is always equal to zero if $y_n = \star$, so that we indeed restrict ourselves to a binary classification problem.

We now restate and prove \cref{lem:oiglo}. 
\oiglo*

\begin{proof}
    Define $m = \ceil{4n / \gamma}$ and thresholds $\tau_i = i\rho$ for $i \in [m]$ and $\rho = 1 / m$. Now, let $\widehat{g}_i$ denote the predictors satisfying the conditions of \cref{lem:oiglob} for the thresholded function classes $\mc{G}_i = \{\psi_{\gamma, \tau_i}(f): f \in \mc{H}\}$. Hence, for a realizable sample $S = ((x_i, f^*(x_i)))_{i = 1}^n$, we have for any $i \in [m]$,
    \begin{equation*}
        \sum_{j = 1}^n \ell_b (\widehat{g}_i(x_j; \Si[j]), \psi_{\gamma, \tau_i} (f^*(x_j))) \leq \fatvf{\mc{H}} (\gamma).
    \end{equation*}
    Our predictor is defined by
    \begin{equation*}
        \widehat{f}(x; S) = \rho \sum_{i = 1}^m \bm{1} \lbrb{\widehat{g}_i (x; S) = 1}.
    \end{equation*}
    Note that for any $z \in [0, 1]$, we have
    \begin{equation*}
        \rho \sum_{i = 1}^m \bm{1} \lbrb{z \geq \tau_i} \leq z \leq \rho + \rho \sum_{i = 1}^m \bm{1} \lbrb{z \geq \tau_i}
    \end{equation*}
    Now, fix $x \in \mc{X}$ and let $\mc{I} = \{i \in [m]: \psi_{\gamma, \tau_i} (f^* (x)) = \star\}$ and $i_l = \min\, \mc{I}$ and $i_h = \max\, \mc{I}$. Noting that for any $i \notin \mc{I}$, it holds that
    \begin{equation*}
        \bm{1} \lbrb{f^*(x) \geq \tau_i} = \bm{1} \lbrb{\psi_{\gamma, \tau_i} (f^*(x)) = 1} \text{ and } \mc{I} = \lbrb{i: i_l \leq i \leq i_h},
    \end{equation*}
    we have
    \begin{align*}
        f^*(x) - \wh{f} (x; S) &\leq \rho + \rho \left(\sum_{i = 1}^{i_l - 1} \bm{1} \lbrb{f^*(x) \geq \tau_i} - \bm{1} \lbrb{\wh{g}_i (x; S) = 1}\right) \\
        &\qquad + \rho \left(\sum_{i = i_l}^{i_h} \bm{1} \lbrb{f^*(x) \geq \tau_i} - \bm{1} \lbrb{\wh{g}_i (x; S) = 1}\right) \\
        &\qquad + \rho \left(\sum_{i = i_h + 1}^m \bm{1} \lbrb{f^*(x) \geq \tau_i} - \bm{1} \lbrb{\wh{g}_i (x; S) = 1}\right) \\
        &\leq \rho + \rho \sum_{i = 1}^m \bm{1} \lbrb{\wh{g}_i (x; S) \neq \psi_{\gamma, \tau_i} (f^*(x)) \text{ and } \psi_{\gamma, \tau_i} (f^*(x)) \neq \star} \\
        &\qquad + \rho \lprp{\sum_{i = i_l}^{i_h} \bm{1} \lbrb{f^*(x) \geq \tau_i} - \bm{1} \lbrb{\wh{g}_i (x) = 1}}  \\
        &\leq \rho + \rho \sum_{i = 1}^m \ell_b (\wh{g}_i (x; S), \psi_{\gamma, \tau_i} (f^*(x))) + \rho \sum_{i = i_l}^{i_h} \bm{1} \lbrb{f^*(x) \geq \tau_i} \\
        &\leq 2\rho + \gamma + \rho \sum_{i = 1}^m \ell_b (\wh{g}_i (x; S), \psi_{\gamma, \tau_i} (f^*(x))).  \end{align*}
    where the last inequality follows from the observation
    \begin{equation*}
        \tau_{i_l} \geq f^*(x) - \gamma \implies \sum_{i = i_l}^{i_h} \bm{1} \lbrb{f^*(x) \geq \tau_i} \leq \frac{\gamma}{\rho} + 1.
    \end{equation*}
    Similarly, for the lower bound, defining $i_m = \max \{i: f^*(x) \geq \tau_i\}$, we have
    \begin{align*}
        f^*(x) - \wh{f} (x; S) &\geq \rho \left(\sum_{i = 1}^{i_l - 1} \bm{1} \lbrb{f^*(x) \geq \tau_i} - \bm{1} \lbrb{\wh{g}_i (x; S) = 1}\right) \\
        &\qquad + \rho \left(\sum_{i = i_l}^{i_h} \bm{1} \lbrb{f^*(x) \geq \tau_i} - \bm{1} \lbrb{\wh{g}_i (x; S) = 1}\right) \\
        &\qquad + \rho \left(\sum_{i = i_h + 1}^m \bm{1} \lbrb{f^*(x) \geq \tau_i} - \bm{1} \lbrb{\wh{g}_i (x; S) = 1}\right) \\
        &\geq -\rho - \rho \sum_{i = 1}^m \bm{1} \lbrb{\wh{g}_i (x; S) \neq \psi_{\gamma, \tau_i} (f^*(x)) \text{ and } \psi_{\gamma, \tau_i} (f^*(x)) \neq \star} \\ 
        &\qquad + \rho \left(\sum_{i = i_l}^{i_h} \bm{1} \lbrb{f^*(x) \geq \tau_i} - \bm{1} \lbrb{\wh{g}_i (x) = 1}\right) \\
        &\geq - \rho - \rho \sum_{i = 1}^m \ell_b (\wh{g}_i (x; S), \psi_{\gamma, \tau_i} (f^*(x))) - \rho (i_h - i_m) \\
        &\geq -3\rho - \gamma - \rho \sum_{i = 1}^m \ell_b (\wh{g}_i (x; S), \psi_{\gamma, \tau_i} (f^*(x))),
    \end{align*}
    where the last inequality follows from
    \begin{equation*}
        \tau_{i_h} - \tau_{i_m} \leq \gamma + \rho \implies i_h - i_m \leq \frac{\gamma}{\rho} + 2.
    \end{equation*}
    From the above two bounds and \cref{lem:oiglob}, we have
    \begin{equation*}
        \sum_{i = 1}^n \ell (\widehat{f}(x; \Si), f^* (x_i)) \leq n \gamma + 3n\rho + \rho m \fatvf{\mc{H}} (\gamma) \leq (n + 1) \gamma + \fatvf{\mc{H}} (\gamma).
    \end{equation*}
    This concludes the proof.
\end{proof}

\subsection{Proof of \cref{thm:loreg}}
\label{ssec:proof_loreg}

\loreg*
\begin{proof}
Without loss of generality, we assume that in the realizable case, it holds that $Y = f^*(X)$ for some $f^* \in \mathcal H$.
Using the convexity of the absolute loss, \cref{thm:main} and \cref{lem:oiglo}, we have, with probability at least $1 - \delta$,
\begin{align*} 
\E_{X} \lsrs{\Bigl|\frac{4}{3n}\sum\nolimits_{t = n/4}^{n - 1}\widehat{f}(X; \Slt{}) - f^*(X)\Bigr|}
&\leq \frac{4}{3n}\sum\nolimits_{t = n/4}^{n - 1}\E_{X} \lsrs{\abs{\widehat{f}(X; \Slt{}) - f^*(X)}}
\\
&\leq \cishaq \lprp{\gamma + \frac{\gamma}{n} + \frac{\fatvf{\mc{H}} (\gamma)}{n} + \frac{1}{n}\log \left(\frac{2}{\delta}\right)}.
\end{align*}
The claim follows if we choose the predictor $\widehat{f}(\cdot) = \frac{4}{3n}\sum\nolimits_{t = n/4}^{n - 1}\widehat{f}(\cdot; \Slt{})$ and note that $\gamma < 1$.
\end{proof}

\section{Proofs for \cref{sec:proof_main_result}: Main result}
\label{sec:proofs_main_result_app}

\subsection{Proof of \cref{lem:chernoff-ish}: Martingale concentration inequalities} 
\label{ssec:app_martingale}

\chernoffish*
\begin{proof}
We first prove \eqref{eq:chernoffish1}.
Fix $\lambda \in (0, 1]$.
We have
\begin{align*}
\E\left[e^{\lambda W_t} \mid \mc{F}_{t-1}\right] &\le \E\left[1 + (e^\lambda-1)W_t  \mid \mc{F}_{t-1}\right] \numberthis{} \label{eq:chernoff1} \\ 
&= 1 +  (e^\lambda-1)\E\left[ W_t \mid \mc{F}_{t-1}\right] \\
&\le \text{exp}((e^\lambda-1)\E[W_t \mid \mc{F}_{t-1}]) \numberthis{} \label{eq:chernoff3},
\end{align*}
where we used the facts that $e^{\lambda x} \le 1 + (e^\lambda-1)x$ for any $0 \le x \le 1$ in \eqref{eq:chernoff1}, and  $1+x \le e^x$ for any $x$ in \eqref{eq:chernoff3}.
Set $\alpha = (e^\lambda-1)$.
We can now easily conclude that
\begin{align*}
\E\left[\prod_{t=1}^T e^{\lambda W_t - \alpha\E[W_t \mid \mc{F}_{t-1}]}\right] &= \E\left[\E\left[\prod_{t=1}^T e^{\lambda W_t - \alpha \E[W_t \mid \mc{F}_{t-1}]} \;\middle|\; \mc{F}_{T-1}\right] \right]\\
&= \E\left[\prod_{t=1}^{T-1}e^{\lambda W_t - \alpha \E[W_t \mid \mc{F}_{t-1}]} \cdot \E\left[e^{\lambda W_T - \alpha \E[W_T \mid \mc{F}_{T-1}]} \;\middle|\; \mc{F}_{T-1}\right]\right]\\
&=\E\left[\prod_{t=1}^{T-1}e^{\lambda W_t - \alpha \E[W_t \mid \mc{F}_{t-1}]} \cdot e^{- \alpha \E[W_T \mid \mc{F}_{T-1}]}\E\left[e^{\lambda W_T } \;\middle|\; \mc{F}_{T-1}\right] \right]\\
&\le\E\left[\prod_{t=1}^{T-1}e^{\lambda W_t - \alpha \E[W_t \mid \mc{F}_{t-1}]} \right] && \text{(using \eqref{eq:chernoff3})}\\
&\le\E\left[e^{\lambda W_1 - \alpha\E[W_1 \mid \mc{F}_{0}]} \cdot 1 \right]  && \text{(using induction)}  \\ 
&\le 1.
\end{align*}
We can thus apply Markov's inequality on the random variable $\prod_{t=1}^T e^{\lambda W_t - \alpha \E[W_t \mid \mc{F}_{t-1}]}$ to conclude that
\begin{align*}
    \Pr\left[ \sum_{t=1}^T \lambda W_t - \alpha\E[W_t \mid \mc{F}_{t-1}] \ge \log(1/\delta)\right] \le \delta.
\end{align*}
\eqref{eq:chernoffish1} now follows after rearranging.

We now prove \eqref{eq:chernoffish2}. 
Fix $\eta \in (0,1]$.
We have
\begin{align*}
\E\left[e^{-\eta W_t} \mid \mc{F}_{t-1}\right] &\le \E\left[1 - (1-e^{-\eta})W_t  \mid \mc{F}_{t-1}\right] \numberthis \label{eq:chernoff4} \\ 
&= 1 -  (1-e^{-\eta})\E\left[ W_t \mid \mc{F}_{t-1}\right] \\
&\le \text{exp}(-(1-e^{-\eta})\E[W_t \mid \mc{F}_{t-1}])\numberthis \label{eq:chernoff6},
\end{align*}
where we used the facts that $e^{-\eta x} \le 1 - (1-e^{-\eta})x$ for any $0 \le x \le 1$ in \eqref{eq:chernoff4}, and  $1+x \le e^x$ for any $x$ in \eqref{eq:chernoff6}.
Set $\beta = (1-e^{-\eta})$.
We can now conclude that
\begin{align*}
\E\left[\prod_{t=1}^T e^{-\eta W_t + \beta\E[W_t \mid \mc{F}_{t-1}]}\right] &= \E\left[\E\left[\prod_{t=1}^T e^{-\eta W_t + \beta \E[W_t \mid \mc{F}_{t-1}]} \;\middle|\; \mc{F}_{T-1}\right] \right]\\
&= \E\left[\prod_{t=1}^{T-1}e^{-\eta W_t + \beta \E[W_t \mid \mc{F}_{t-1}]} \cdot \E\left[e^{-\eta W_T + \beta \E[W_T \mid \mc{F}_{T-1}]} \;\middle|\; \mc{F}_{T-1}\right]\right]\\
&=\E\left[\prod_{t=1}^{T-1}e^{-\eta W_t + \beta \E[W_t \mid \mc{F}_{t-1}]} \cdot e^{\beta \E[W_T \mid \mc{F}_{T-1}]}\E\left[e^{-\eta W_T } \;\middle|\; \mc{F}_{T-1}\right] \right]\\
&\le\E\left[\prod_{t=1}^{T-1}e^{-\eta W_t + \beta \E[W_t \mid \mc{F}_{t-1}]} \right] && \text{(using \eqref{eq:chernoff6})}\\
&\le\E\left[e^{-\eta W_1 + \beta\E[W_1 \mid \mc{F}_{0}]} \cdot 1 \right] && \text{(using induction)}  \\ 
&\le 1.
\end{align*}
We can thus apply Markov's inequality on the random variable $\prod_{t=1}^T e^{-\eta W_t + \beta \E[W_t \mid \mc{F}_{t-1}]}$ to conclude that
\begin{align*}
    \Pr\left[ -\sum_{t=1}^T \eta W_t + \beta\E[W_t \mid \mc{F}_{t-1}] \ge \log(1/\delta)\right] \le \delta.
\end{align*}
\eqref{eq:chernoffish2} now follows after rearranging.
\end{proof}
\end{document}